\documentclass[journal]{IEEEtran}
\usepackage{amsmath,amssymb,amsthm}
\usepackage{graphicx}
\usepackage{epstopdf}
\usepackage{cite}
\usepackage{algorithm}
\usepackage{algpseudocode}
\usepackage{comment}

\usepackage{caption}
\usepackage{subcaption}
\usepackage{float}
\usepackage{bm}
\usepackage{tikz}
\usepackage{pgfplots}
\usepgfplotslibrary{fillbetween}
\usepgfplotslibrary{groupplots}
\pgfplotsset{compat=1.18} 

\usepackage{pgfplots, pgfplotstable}
\usetikzlibrary{pgfplots.groupplots, positioning} 

\IEEEoverridecommandlockouts
 
\theoremstyle{plain}
\newtheorem{theorem}{Theorem}

\newtheorem{lemma}{Lemma}
\newtheorem{proposition}{Proposition}

\theoremstyle{definition}

\newtheorem{example}{Example}

\theoremstyle{remark}

\DeclareMathOperator*{\argmin}{arg\,min}

\DeclareMathOperator{\mbE}{\mathbb{E}}
\newcommand{\reals}{\mathbb{R}}

\newcommand{\ccalN}{\mathcal N}

\newcommand{\avt}{\frac{1}{T} \sum_{t=1}^T}

\newcommand{\feasible}{\mathcal C_{CL}(\epsilon)}
\newcommand{\bbI}{\mathbb {I}}
\usepackage{ifthen}
\newboolean{showcomments}
\setboolean{showcomments}{false}
\usepackage{todonotes}
\newcommand{\juanb}[1]{  \ifthenelse{\boolean{showcomments}}
{\todo[inline,color=pink]{Juan B: #1}}{}}
\newcommand{\juanc}[1]{  \ifthenelse{\boolean{showcomments}}
{\todo[inline,color=cyan]{Juan C: #1}}{}}
\newcommand{\miguel}[1]{  \ifthenelse{\boolean{showcomments}}
{\todo[inline,color=yellow]{Miguel: #1}}{}}

\newcommand{\thetacl}{\theta_t^\dag}
\newcommand{\thetac}{\hat\theta_{c}}
\newcommand{\thetastar}{\theta_t^\star}
\newcommand{\thetahat}{\hat\theta_t}
\newcommand{\thetage}{\theta_g^\dag}

\begin{document}

\title{
Cross-Learning from Scarce Data via Multi-Task Constrained Optimization
}
\author{Leopoldo Agorio, Juan Cervi\~no,  Miguel Calvo-Fullana,  Alejandro Ribeiro, and Juan Andr\'es Bazerque.
\thanks{
This work was supported in part by Spain's Agencia Estatal de Investigaci\'on under grant RYC2021-033549-I.
}
\thanks{
L. Agorio and J. A. Bazerque are with the Department of Electrical Engineering, School of Engineering, Universidad de la Rep\'ublica, Montevideo, 11300, Uruguay (e-mail: \mbox{lagorio@fing.edu.uy,jbazerque}@fing.edu.uy) 
}
\thanks{
J. Cervi\~no is with the Massachusetts Institute of Technology, Cambridge, MA 02139, USA (e-mail: \mbox{jcervino}@mit.edu).
}
\thanks{
M. Calvo-Fullana is with the Department of Engineering, Universitat Pompeu Fabra, Barcelona, 08018, Spain (e-mail: miguel.calvo@upf.edu).
}
\thanks{
A. Ribeiro is with the Department of Electrical and Systems Engineering, University of Pennsylvania, Philadelphia, PA 19104, USA (e-mail: \mbox{aribeiro}@seas.upenn.edu).
}
}

\maketitle
\begin{abstract}

A learning task, understood as the problem of fitting a parametric model from supervised data, fundamentally requires the dataset to be large enough to be representative of the underlying distribution of the source. When data is limited, the learned models fail generalize to cases not seen during training. This paper introduces a multi-task \emph{cross-learning} framework to overcome data scarcity by jointly estimating \emph{deterministic} parameters across multiple, related tasks. We formulate this joint estimation as a constrained optimization problem, where the constraints dictate the resulting similarity between the parameters of the different models, allowing the estimated parameters to differ across tasks while still combining information from multiple data sources. This framework enables knowledge transfer from tasks with abundant data to those with scarce data, leading to more accurate and reliable parameter estimates, providing a solution for scenarios where parameter inference from limited data is critical. We provide theoretical guarantees in a controlled framework with Gaussian data, and show the efficiency of our cross-learning method in applications with real data including image classification and propagation of infectious diseases.   
\end{abstract}


\IEEEpeerreviewmaketitle

\section{Introduction}
\label{sec:Introduction}

The machine learning problem, in general, involves extracting information from a dataset, which is typically achieved by fitting the parameters of a model \cite{1Hastie2009}, whether it be a neural network or a more specific parametric function that incorporates additional knowledge about the data source. Once fitted, this parametric model can be used for classification, prediction, or estimation, 
 serving various purposes. For example, one might want to classify input signals, as in the case of a robot or autonomous vehicle that captures an image and needs to detect whether there is an obstacle to avoid \cite{2Teichmann2018}. One might also want to predict, for instance, the solar energy generation capacity of a power system, in order to set prices in a day-ahead market \cite{3Lan2019}. In some cases, estimating the parameters themselves constitutes the learning objective. For example, an infection propagation model, such as for influenza or COVID, can be used to predict the number of infected individuals, but the parameters themselves contain information about how the population behaves and which policies are most effective \cite{4Dorr2019}.

Appropriate parameters are typically found by minimizing a loss or fit function, which measures how well the model explains the data for a given parameter set. Beyond this, a fundamental design decision is whether to assume the parameters are deterministic or random, a choice that opens two distinct methodological paths \cite{5Kay1993}. In the Bayesian model, parameters are treated as random variables governed by a prior distribution, which may in turn be modeled by a second layer of hyperparameters. In the contrasting deterministic paradigm, the parameters are assumed to be fixed, unknown quantities that define the data's underlying distribution.
Both the probabilistic and deterministic approaches are valid. The choice between them often depends on one's confidence in a potential prior distribution. For instance, if parameters are known to lie within a certain interval, imposing a uniform prior distribution  over that interval can improve estimation, especially in a low-data regime. However, if this prior information is incorrect, it will introduce bias, degrading the estimation quality.
In this paper, we adopt the deterministic parameter paradigm. We do not assume a prior distribution but instead seek to incorporate additional information to improve performance, particularly when working in a \emph{low-data regime}, as it the norm in diverse fields such as medical data \cite{9Zhou2021}, speech recognition \cite{10Cai2021}, and anomaly detection \cite{11Georgescu2021}.

Our  approach sits at the intersection of multi-task learning \cite{6Caruana1997}, multi-agent systems, and meta-learning \cite{7Finn2017,8Chua2021}. Multi-task learning has the potential of augmenting the data by addressing different tasks together, incorporating data from multiple sources or domains and combining it into  a joint learning problem.  If these sources are related, then they cross-fertilize to improve performance. 
Such data augmentation is particularly useful when data from all sources are scarce, or in cases in which there is a particular incipient source which has not produced enough data yet. This bring us to the concept of meta-learning, which seeks to determine whether models trained on past data are useful for explaining future data—especially when that future data comes from a previously unseen source. Through meta-learning, we can use past experience to learn quickly in a new situation, combining information from similar past cases with the limited data available in the new context. The intuition from multi-task and meta-learning leads us to propose \emph{cross-learning} as a methodology for combining the training from multiple sources.

The central challenge lies in \emph{what} to share between sources and \emph{how} to share it \cite{13Zhang2017,14Vafaeikia2020}. Various approaches have been proposed. In the deep learning context, some studies explore sharing common data representations, which often translates to sharing the early layers of a neural network \cite{15Maurer2016,16Yang2016}. Other works have developed deep architectures with specific blocks shared across tasks \cite{17Misra2016,18Wallingford2022}, often empirically determining which shared blocks yield the most significant impact. In these approaches, part of the architecture is common, while another part remains private to each task. Measuring task-relatedness is also an active research area in the field of multi-task learning. Some work focuses on grouping tasks to learn them jointly within each group \cite{19Standley2020}. Other methods evaluate task differences by measuring discrepancies between their learned models \cite{20Ma2018}, or use a convex surrogate of the learning objective to model task relationships \cite{21Zhang2012}. In contrast to Bayesian formulations that model task-relatedness via shared priors \cite{22Bonilla2007,23Hensman2013}, our approach does not assume any prior structure on the learned functions.

In this paper, we formalize and extend cross-learning for supervised learning, a methodology based on a constrained optimization framework we originally applied in the reinforcement learning paradigm \cite{cervino2019meta,cervino2020multi}. The core idea is to manage the bias-variance trade-off in multi-task settings explicitly. Our preliminary work \cite{cervino2021multi} drew intuition from the two extremes of multi-task learning: (i) training separate, task-specific estimators (using only local data, which can suffer from high variance) and (ii) training a single, consensus model on all data (which can introduce high bias if tasks are dissimilar). Cross-learning finds a balance between these two extremes by allowing task-specific models while imposing constraints that keep them close to each other, thereby controlling the bias-variance trade-off. This work builds on that foundation with several key contributions. First, we provide a theoretical proof that the cross-learning formulation successfully exploits this trade-off, outperforming both the naive separate and consensus approaches. Second, we generalize the method to support arbitrary functional constraints, lifting intuition from \cite{cervino2023multitradeoff}. This is a critical extension, particularly for neural networks, where models with close parameters can still produce highly distant output distributions. By imposing constraints directly on the model outputs rather than the parameter space, we can more effectively regularize the model's behavior. Third, we develop new algorithms to solve these optimization problems in the dual domain. Finally, we validate our theoretical findings and algorithmic performance through experiments on two distinct problems. The first involves fitting an infection model using data from the COVID-19 pandemic, treating data from different countries as separate but related tasks. The second is an image classification task to distinguish objects using distinct representations of the same object. The experimental results confirm that cross-learning consistently outperforms the alternative approaches of training completely separate models or merging all data into a single consolidated model.


\section{Multi-task learning}
\label{sec:MTL}
We consider the problem of jointly optimizing a set of parametric functions, each solving a regression or classification task over a separate dataset.  These datasets are collected from different but \emph{related} sources; therefore, we seek a method to fit the \emph{deterministic} parameters of the associated functions jointly.  Formally, consider a finite set of tasks $t\in[1,\dots,T]$, and let the parametric function $f:\mathcal{X}\times\Theta\to \mathcal{Y}$ be the map between the input space $\mathcal{X}\subset\mathbb{R}^P$, and output space $\mathcal{Y}\subset\mathbb{R}^Q$ parameterized by $\theta\in \mathbb{R}^S$.  Our goal is to find the parameterizations that minimize the expected loss  $\ell:\mathbb{R}^Q\times \mathbb{R}^Q \to \reals^+$ over the probability distribution $p_t(x,y)$,
\begin{align}\label{prob:MTL_statistical}
	\{\thetastar\}_{t=1}^T=\underset{\{\theta_t\}_{t=1}^T}{\arg\min}  \quad   & \frac{1}{T}\sum_{t=1}^T  \mathbb{E}_{p_t(x,y)}[ \ell\left(y,f(x,\theta_t)\right) ]. 
	\tag{$\text{P}_{\text{ML}}$}
\end{align}

 We do not assume, however, to have access to the distributions $p_t(x,y)$. Instead, we align with the multi-task learning literature, assuming that each task $t$ is equipped with a dataset $\mathcal{D}_t$, containing $N_t$ samples  $(x_i,y_i)\in(\mathcal{X},\mathcal{Y}),i=1,\dots,N_t$. The datasets $\mathcal{D}_t$ are assumed to be drawn according to the unknown joint probabilities $p_t(x,y)$. The empirical version of the multi-task learning problem \eqref{prob:MTL_statistical} can thus be written as,
\begin{align}\label{prob:MTL_independent_EMP}
	\{\thetahat\}_{t=1}^T=\underset{\{\theta_t\}_{t=1}^T}{\arg \min}  \quad   &\frac{1}{T}\sum_{t=1}^T \frac{1}{N_t} \sum_{i=1}^{N_t} \ell\left(y_i,f(x_i,\theta_t)\right). 
	\tag{$\text{P}_\text{S}$}
\end{align}
The sole difference between the empirical multi-task learning problem \eqref{prob:MTL_independent_EMP} and its statistical counterpart \eqref{prob:MTL_statistical} is the fact that the expectations over $p_t(x,y)$ have been replaced by an empirical sum over $\mathcal{D}_t$. Under a deterministic model for the parameters $\theta_t$, this empirical problem \eqref{prob:MTL_independent_EMP}  is \emph{separable} across tasks. However, given that the number of samples $N_t$ is finite, solving for the parameters $\hat\theta_t$ separately can be detrimental. If there is correlation across tasks, solving them separately can lead to a loss of information, as data from one task could be beneficial for learning others. Such separate estimation defeats the primary purpose of multi-task learning, which is to learn tasks jointly.
To leverage the joint information from different sources, a simple approach is to merge all datasets and learn a common solution for all tasks, or equivalently, imposing a consensus constraint $\theta_t=\theta$ for all $t$, forcing all tasks to be modeled by the same parameter $\theta_c$, solution to 
\begin{align}\label{prob:MTL_centralized_EMP}
	\thetac= \underset{\theta}{\arg\min}  \quad   &  \frac{1}{T}\sum_{t=1}^T \frac{1}{N_t} \sum_{i=1}^{N_t} \ell\left(y_i,f(x_i,\theta)\right).
	\tag{$\text{P}_\text{C}$}
\end{align}
This \emph{consensus} approach \eqref{prob:MTL_centralized_EMP} does provide a joint solution to our multi-task learning problem, merging data from all sources. However, forcing a strict consensus might be too restrictive in practice, as a solution that is good for all tasks might not be the best for any individual task. Hence, we explore a more balanced approach that trades off between learning individual parameters $\hat{\theta}_t$, as in the \emph{separable} problem \eqref{prob:MTL_independent_EMP}, and a common parameter $\hat{\theta}_c$ as in the  \emph{consensus}  multi-task learning problem \eqref{prob:MTL_centralized_EMP}. 
\subsection{Multi-Task Learning Bias-Variance Trade-Off}

To motivate our proposed approach, let us present a reductive yet informative example to illustrate how bias and variance affect the solutions to \eqref{prob:MTL_independent_EMP} and \eqref{prob:MTL_centralized_EMP}.

\begin{example}\label{ex:bias-variance}
Consider the task of recovering a parameter $\thetastar\in\reals^d$ from  samples corrupted by additive, zero-mean, i.i.d. noise
\begin{align}\label{eq:noisy_data}
    y_{tn}=\thetastar+\eta_{tn},\ n=1,\ldots,N_t
\end{align}
with $\eta_{tn}\sim\ccalN(0,\sigma \bbI)$. To estimate $\thetastar$ from the samples $y_{tn}$, we minimize the mean square error as follows, 
\begin{align}\label{eqn:hatP_indepedent}
   \thetahat &= \argmin_{\theta} \frac{1}{N_t} \sum_{n=1}^{N_t} ||y_{tn}-\theta ||^2 = \frac{1}{N_t}\sum_{n=1}^{N_t} y_{tn}. 
   	\tag{$\bar{\text{P}}_{\text{S}}$}
\end{align}
The estimator $\thetahat$ in \eqref{eqn:hatP_indepedent}, is a particular case of the  separable estimator \eqref{prob:MTL_independent_EMP} with a constant regressor  $f(x,\theta)=\theta$ and square loss $\ell(y,f(x,\theta))=\|y-\theta\|^2$. If the true parameters $\{\thetastar\}$ are close to one another, it might be beneficial to consider a single estimator based on all available samples,
\begin{align}\label{eqn:hatP_centralized}
    \thetac &= \argmin_{\theta} \sum_{t=1}^T \frac{1}{N_t} \sum_{n=1}^N ||y_{tn}-\theta ||^2,
   	\tag{$\bar{\text{P}}_{\text{C}}$}
\end{align}
which is the simplified version of \eqref{prob:MTL_centralized_EMP} for the model in \eqref{eq:noisy_data}. Notice that in this reduced case, the consensus estimator admits a closed-form solution given by the average of all samples, which can be expressed as the average $ \thetac 
= \frac{1}{T}\sum_{t=1}^T \hat \theta_t.$ of the separate estimates $\hat \theta_t$ in \eqref{eqn:hatP_indepedent},
Since the noise is zero-mean, the estimators in \eqref{eqn:hatP_indepedent} are unbiased, and by virtue of the i.i.d. model, their variance reduces with the number of samples according to
\begin{align}\label{eq:var_thetahat}
 \text{var}(\hat \theta_t)=\mathbb{E}\left[\left\|\hat \theta_t-\mathbb{E}[\hat \theta_t]\right\|^2\right]= \frac{d}{N_t}\sigma^2.
 \end{align}
On the other hand, we show in Appendix \ref{app:bias-variance}, that the consensus estimator has variance 
$
\text{var}(\thetac)=\frac{1}{T^2}\sum_{t=1}^{T}\frac{d}{N_t} \sigma^2 $,
which simplifies if all datasets have the same size $N_t=N$,   
\begin{align}
\label{eq:var_thetac_simplified}
\text{var}(\thetac)=\frac{1}{T}\frac{d}{N} \sigma^2.
\end{align}
Intuitively, pooling data from all tasks increases the effective dataset size from $N_t=N$ in \eqref{eq:var_thetahat} to $TN$ in \eqref{eq:var_thetac_simplified}, thereby reducing the uncertainty in the average estimate. However, this reduction in variance comes at the cost of introducing bias:
\begin{align*}
\mathbb{E}\left[\hat \theta_c - \thetastar \right] 
= \frac{1}{T}\sum_{\tau=1}^T E[\hat \theta_\tau] -\thetastar
= \frac{1}{T}\sum_{\tau=1}^T  (\theta_\tau^\star-\thetastar )\neq 0.
\end{align*}
\end{example}

In the light of this basic example, we propose the following \emph{cross-learning} estimator that balances bias and variance by lying in between the fully separable solution \eqref{prob:MTL_independent_EMP} and the consensus counterpart \eqref{prob:MTL_centralized_EMP}. This approach allows the estimated parameters to differ across tasks while still combining information from multiple data sources,
\begin{align}
	\{\thetacl\},\thetage=\underset{\{\theta_t\},\theta_g}{\arg\min}   &  \frac{1}{T}\sum_{t=1}^T\frac{1}{N_t}\sum_{i=1}^{N_t} \ell \left(y_i,f(x_i,\theta_t)\right) \label{prob:CL} \tag{$\text{P}_{\text{CL}}$}
	\\
	\text{subject to:}
	\quad   &  \| \theta_t-\theta_g\| \leq \epsilon, \quad t=1,\ldots,T.\nonumber
\end{align}

The \emph{cross-learning} problem \eqref{prob:CL} introduces a task-specific parameter $\thetacl$ for each task and a global parameter $\thetage$ shared among all tasks. This formulation connects the separable multi-task learning problem \eqref{prob:MTL_independent_EMP} and the consensus approach \eqref{prob:MTL_centralized_EMP} by imposing a constraint on the closeness of the task-specific parameters to the global parameter. This constraint relaxes the strict equality of the consensus problem, allowing the solutions for different tasks to vary.

 The centrality parameter $\epsilon$ controls the degree of similarity between the task-specific solutions. A larger $\epsilon$ allows solutions to be more task-specific, while a smaller $\epsilon$ encourages them to be closer to the global solution. Note that, by enforcing $\epsilon=0$, all solutions are forced to be equal and \eqref{prob:CL} becomes equivalent to enforce consensus in \eqref{prob:MTL_centralized_EMP}. Conversely for a sufficiently large $\epsilon$, the constraint becomes inactive, and \eqref{prob:CL} is equivalent to the separable problem \eqref{prob:MTL_independent_EMP}.
 
The advantages of solving problem \eqref{prob:CL} are twofold. First, unlike the separable problem it combines information across tasks through the constraint, thereby exploiting data from related tasks. Second, unlike the consensus approach, it allows for task-specific solutions, leading to better performance on individual tasks. As we will show, cross-learning controls the bias-variance trade-off by enforcing the task-specific solutions to be close. We will formally prove that for the simplified model in Example \ref{ex:bias-variance}, there exists a value of $\epsilon$ for which our cross-learning estimator achieves a guaranteed improvement in mean squared error compared to both the separable and consensus estimators. And  we will  demonstrate that these improvements generalize to more complex scenarios in experiments with real data.

\section{Performance Analysis}
\label{sec:perf_analysis}

Throughout this section, we will consider the simplified model given by \eqref{eq:noisy_data}, with the associated regressor $f(x,\theta)=\theta$ and the quadratic loss $\ell(y,f(x,\theta))=\|y-\theta\|^2.$ Under these assumptions, $\thetahat=(1/N_t)\sum_{n=1}^{N_t} y_{nt}$ is a sufficient statistic and the cross-learning estimator in \eqref{prob:CL} simplifies to
\begin{align}\label{eqn_perf_ana:cross_learning}
   \{\thetacl \},\thetage&= \argmin_{\theta_t,\theta_g}\sum_{t=1}^{T} ||\thetahat-\theta_t ||^2    	\tag{$\bar{\text{P}}_{\text{CL}}$}\\
   &\text{subject to } || \theta_t-\theta_g||\leq \epsilon, \quad t=1,\ldots,T. \nonumber
\end{align}
 We will compare this estimator \eqref{eqn_perf_ana:cross_learning} to its fully separable \eqref{eqn:hatP_indepedent} and strict consensus \eqref{eqn:hatP_centralized} alternatives using the mean squared error as a performance metric, for which we define
\begin{align}
\mathcal E_S&=  \frac{1}{T}\sum_{t=1}^T \left\|\thetahat-\thetastar\right\|^2 , \label{eqn_mseI_def}\\
\mathcal E_C&= \frac{1}{T}\sum_{t=1}^T \left\|\thetac-\thetastar\right\|^2 , \label{eqn_mseC_def}\\
\mathcal E_{CL}(\epsilon) &=  \frac{1}{T}\sum_{t=1}^T \left\|\thetacl(\epsilon)-\thetastar\right\|^2 ,\label{eqn_mseCL_def}
\end{align}
for the errors resulting from the separable  \eqref{eqn:hatP_indepedent}, consensus \eqref{eqn:hatP_centralized}, and cross-learning \eqref{eqn_perf_ana:cross_learning} estimators, respectively.
The error is measured with respect to the true data-generating parameters $\thetastar$. In what follows, we show that by appropriately tuning the centrality parameter $\epsilon$, the error of our cross-learning estimator \eqref{eqn_perf_ana:cross_learning} can always be made smaller than the error produced by both the separate  \eqref{eqn:hatP_indepedent} and consensus estimators \eqref{eqn:hatP_centralized}. 

To begin with, we compare the cross-learning estimator to the consensus one. We note that $\mathcal E_{CL}(\epsilon)=\mathcal E_{C}$ when $\epsilon=0$ since in this case the constraint in \eqref{eqn_perf_ana:cross_learning} imposes consensus. If we define the difference in error as $\Delta \mathcal E(\epsilon):=\mathcal E_{CL}(\epsilon)-\mathcal E_C$, which is zero at $\epsilon=0$, it is sufficient to show that the derivative of $\mathbb{E}\left[\Delta\mathcal E(\epsilon)\right]$ is strictly negative at $\epsilon=0$. This would guarantee the existence of an $\epsilon >0$ such that $\mathbb{E}\left[\Delta\mathcal E(\epsilon)\right]<0$, meaning that the mean squared error of the cross-learning estimator is strictly lower than the consensus alternative, i.e. $\mathbb{E}\left[\mathcal E_{CL}(\epsilon)\right]<\mathbb{E}\left[\mathcal E_{C}\right]$. The following proposition proves that the derivative of $\mathbb{E}\left[\Delta\mathcal E(\epsilon)\right]$ is indeed strictly negative.
\begin{proposition}\label{prop:centralized_vs_cl}
The cross-learning estimator \eqref{eqn_perf_ana:cross_learning}, for an $\epsilon$ approaching zero, achieves a strictly lower mean squared error than the consensus estimator \eqref{eqn:hatP_centralized}. That is
\begin{align}\label{prop_cl_better_than_c}
\lim_{\epsilon \to 0^+}&\frac{1}{\epsilon}\left(\mathbb{E}\left[\mathcal E_{CL}(\epsilon)-\mathcal E_{C}\right]\right)<0.  
\end{align}
\end{proposition}
\begin{proof}
Leveraging Lemma \ref{lemma:gradientMSE} in Appendix \ref{app:proofs} we have 
\begin{align}
    \lim_{\epsilon \to 0^+}&\frac{1}{\epsilon}\left(\mathbb{E}\left[\mathcal E_{CL}(\epsilon)-\mathcal E_{C}\right]\right)=-\frac{2}{T}\sum_{t=1}^T (\thetastar-\theta_c^\star)^\top \mathbb{E}[ \hat u_t] \label{eq:fromlemma4}
\end{align}
where we defined  $\hat u_t=(\thetahat-\thetac)/\|\thetahat-\thetac\|$ to simplify notation. 

We will prove that 
    $\mathbb{E}[\hat u_t]=a_t (\thetastar-\theta^\star_c)$
with $a_t>0$, strictly. After doing so,  we will be able to substitute it in \eqref{eq:fromlemma4} and obtain the desired sign for 
\begin{align}
\lim_{\epsilon \to 0^+}&\frac{1}{\epsilon}\left(\mathbb{E}\left[\mathcal E_{CL}(\epsilon)-\mathcal E_{C}\right]\right)
    =    -2\sum_t a_t \|\thetastar-\theta^\star_c\|^2<0
\end{align}

In order to prove $\mathbb{E}[\hat u_t]=a_t (\thetastar-\theta^\star_c)$, let us analyze the distribution of the difference $\thetahat-\thetac$. With $\thetahat=\thetastar+n_t$ and $\thetac=\theta_c^\star+(1/T)\sum_{t=1}^T n_t$ we can write 
\begin{align}
    \thetahat-\thetac&=\thetastar-\theta^\star_c + \left(1-\frac{1}{T}\right)n_t -\frac{1}{T}\sum_{k\neq t} n_k
\end{align}
If $n_t$ are independent zero-mean Gaussian noise vectors with covariance $\sigma^2 I$ for all $t=1,\ldots,T$ we conclude that 
$   \thetahat-\thetac=\thetastar-\theta^\star_c + v_t$,\
with zero-mean Gaussian noise $v_t\sim \mathcal N(0, \sigma_v^2 \bbI)$ of variance  $\sigma_v^2:=(T-1)\sigma^2/T$.
Under this noise distribution, we can proceed to compute 
\begin{align}\label{eqn:expected_value_inner_prod}
    \mbE[\hat u_t]&=\frac{1}{(2\pi\sigma_v^2)^{n/2}}\int_{\reals^n} \hat u_t e^{-||v_t||^2/(2\sigma_v^2)}dv_t.
\end{align}
where $\hat u_t$  is a function of the integration variable  $v_t$ given by  $\hat u_t=(\thetastar-\theta^\star_c+v_t)/ \|\thetastar-\theta^\star_c+v_t\|$.
\begin{figure}[t]
    \centering
	\includegraphics[scale=1]{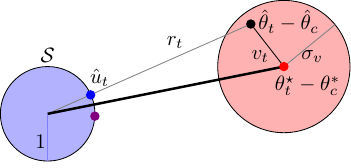}    
    \caption{Intuitive interpretation of proof and the symmetry of $\hat u_t=(\thetahat-\thetac)/\|\thetahat-\thetac\|$ on step \eqref{eq:ptofs} of the proof of Proposition \ref{prop:centralized_vs_cl}. The vector $\hat\theta_t-\hat\theta_c$ admits two representations, one as $\theta_t^\star-\theta_c^\star+v_t$ and a second one as $r_t u_t$, with $r_t=  \|\hat\theta_t-\hat\theta_c\|$, which leads to the change of variables $v_t=r_t u_t-\theta_t^\star-\theta_c^\star$. Furthermore, vectors $\hat u_t$ come in pairs. That is, for each blue point representing the vector $\hat u_t$, there is a purple mirror image across  $\theta_t^\star-\theta_c^\star$  such that  its inner product with $\theta_t^\star-\theta_c^\star$ has same magnitude but opposite sign. }
    \label{fig:proof_symmetry} 
\end{figure}
To evaluate this integral, we introduce the  polar change of variables  illustrated in Fig. \ref{fig:proof_symmetry}. Specifically,  we express the vector $\hat\theta_t-\hat\theta_c=\thetastar-\theta^\star_c+v_t=r_t  u_t$  in polar form with the variable  $r_t\in\reals, \ r_t\geq 0$ representing its norm, and $ u_t$ representing the phasor in the unit sphere $\mathcal S$. Notice that by definition,  $\hat u_t$ is unit norm and points in the direction of $\hat\theta_t-\hat\theta_c$. Thus, we can identify $u_t=\hat u_t$ and express the change of variables as 
   $ v_t= r_t  \hat u_t -(\thetastar-\theta^\star_c)$ with $ 
    r_t \in(0,\infty),$   and $\hat u_t \in \mathcal S:= \{u\in \mathbb R^n: ||u||=1\}$, and rewrite 
\begin{align}
\mbE[\hat u_t]&=\hspace{-3pt}\int_{\mathcal S}\int_{r=0}^\infty\frac{ur^{n-1}}{(2\pi\sigma_v^2)^{n/2}} e^{-\frac{\| r u-(\thetastar-\theta^\star_c)\|^2}{2\sigma_v^2}}drdu=\hspace{-3pt}\int_{\mathcal S}\hspace{-5pt}up(u) du \label{eq:polars}
\end{align}
where we incorporated the Jacobian  $|J|=  r_t^{n-1}$, and defined 
 $p(u)$ as
\begin{align}
    p(u)&=\int_{r=0}^\infty \frac{r^{n-1} \sigma_v}{(2\pi\sigma_v^2)^{n/2}} e^{-\frac{\| r u-(\thetastar-\theta^\star_c)\|^2}{2\sigma_v^2}} dr\\
    &=\int_{r=0}^\infty \frac{ r^{n-1}e^{-q(r)}}{(2\pi\sigma_v^2)^{n/2}}  e^{\frac{r}{\sigma_v}(\thetastar-\theta^\star_c)^\top u} dr\label{eq:ptofs} 
\end{align}
and substituted $q(r):=(r^2 \sigma_v^2  + \|\thetastar-\theta^\star_c\|^2)/(2\sigma_v^2)$. We finish the proof with a symmetry argument. Notice that $p(u)$ in \eqref{eq:ptofs} depends on $u$ through the inner product between $(\thetastar-\theta^\star_c)$ and $u$. Therefore two vectors $u$  with opposite  angles to $(\thetastar-\theta^\star_c)$ (as those depicted in blue and purple in Fig. \ref{fig:proof_symmetry}) receive opposite weights $p(u)$ in the integral \eqref{eq:polars}.  That means that aggregate  contribution of these two mirrored  points to  \eqref{eq:polars} is a vector in the same  direction as $\thetastar-\theta_c^\star$.  Since the sphere $\mathcal S$ has axial symmetry around $\theta_t^\star-\theta_c^\star$, all $u\in \mathcal S$ has a mirrored image around $\theta_t^\star-\theta_c^\star$. Thus, the integral $\int_{\mathcal S}up(u) du$ in  \eqref{eq:polars} is a vector pointing in the direction of  $\theta_t^\star-\theta_c^\star$. It follows that $\mathbb{E}[\hat u_t]$ must be colinear with  $(\thetastar-\theta^\star_c)$, i.e., $\mathbb{E}[\hat u_t] =a_t (\thetastar-\theta^\star_c)$ with $a_t \in \mathbb{R}$. Furthermore, according to \eqref{eq:ptofs} vectors $u$ with acute angles to $(\thetastar-\theta^\star_c)$ have  heavier weights in the integral,  hence $a_t>0$, strictly, as we wanted to prove.
\end{proof}

Proposition \ref{prop:centralized_vs_cl} is relevant to our analysis in the sense that  it guarantees the existence of a value $\epsilon>0$ for which the cross-learning estimator yields a lower mean-squared error than the consensus one. Indeed, with   consensus and cross-learning estimators being equivalent for $\epsilon=0$,  \eqref{prop_cl_better_than_c} implies that by increasing $\epsilon$ slightly, we obtain better performance.

Now that we have established that the cross-learning estimator can outperform  the consensus one, it remains to show that it also improves with respect to  the separable estimator. For this purpose, we establish two bounds for $\mathcal E_{CL}(\epsilon)$ in terms of $\mathcal E_{S}$.  The first one, in Proposition \ref{prop:independent_vs_cl_estimator}, states that if we choose $\epsilon$  large enough, then $\mathcal E_{CL}(\epsilon)\leq \mathcal E_{S}$. 

\begin{proposition}\label{prop:independent_vs_cl_estimator}
Let $\{y_{tn}\}_{n=1}^{N_t}$ be the dataset associated with task $t=1,\ldots, T$, and $ \thetahat=\frac{1}{N_t} \sum_{i=1}^{N_t} y_{tn}$ its  sample average. If $\epsilon$ is chosen large enough to satisfy  
\begin{align}\label{eq:max_max_dif}
\epsilon\geq \max_{t}\max_\tau \|\thetastar-\hat\theta_\tau\|,
\end{align}
then the deterministic error of the cross-learning estimator is smaller than the error of the separable estimator, i.e., \begin{align} \mathcal E_{CL}(\epsilon)-\mathcal E_{S}\leq 0.\label{eq:better_independent}
\end{align}
\end{proposition}

\begin{proof}
We will start by showing that \eqref{eq:max_max_dif} implies that the generators $\thetastar$ belong to the ball of center $\thetage$ and radius $\epsilon$, which we denote by $\mathcal{B}(\thetage,\epsilon)$, i.e., 
\begin{align}\label{eqn:agnostic_condition}
        ||\thetastar-\thetage||\leq\epsilon, \text{ for all }t.
\end{align}
For that purpose, we consider Lemma \ref{lemma:thetaG} in Appendix \ref{app:proofs}, which establishes that  $\thetage$ can be written as  a convex combination $\thetage = \sum_{t=1}^{T} \gamma_t \thetahat$ of the separable estimators, 
with coefficients $\gamma_t\geq 0$ such that $\sum_{t=1}^T \gamma_t=1$. It follows $\thetastar-\thetage=\sum_{\tau=1}^T \gamma_\tau (\thetastar-\hat \theta_\tau)$, and thus 
\begin{align}
||\thetastar- \thetage||&\leq\hspace{-2pt} \sum_{\tau=1}^T\hspace{-2pt}\gamma_\tau ||\thetastar- \hat \theta_\tau||\leq\hspace{-8pt} \max_{\tau=1,\ldots,T}||\thetastar- \hat \theta_\tau||\leq \epsilon,\label{eq:thruth_in_ball}
\end{align}
so that \eqref{eqn:agnostic_condition} holds for all $t=1,\ldots, T$.

\begin{figure}
	\centering
	\includegraphics[scale=1]{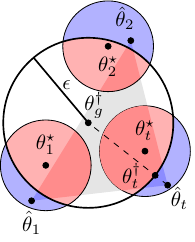}    
	\caption{Geometric argument for the proof of Proposition \ref{prop:independent_vs_cl_estimator}. The generators $\thetastar$ are inside the ball $\mathcal B(\thetage, \epsilon)$. The separate estimates $\thetahat$ are projected into $\mathcal B(\thetage, \epsilon)$ resulting in the cross-learning estimator $\thetacl$. Points $\thetahat$ in the red region will not be changed, so that $\thetahat=\thetacl$, and points in the blue region will be projected to the surface where the error to $\thetastar$ will be lower. }
	\label{fig:prueba_mse_projection}
\end{figure}
 Next, notice that by virtue of Lemma \ref{lemma:projection} in Appendix \ref{app:proofs}, the cross-learning estimator $\thetacl$ is the projection  $\thetacl=P_{\mathcal{B}}(\thetahat)$  of $\thetahat$ to the ball $\mathcal{B}:=\mathcal{B}(\thetage,\epsilon)$ of radius $\epsilon$ and center $\thetage$ (see Figure \ref{fig:prueba_mse_projection}). We have shown in \eqref{eq:thruth_in_ball} that all $\thetastar$  belong to $\mathcal{B}$. Thus, we have $P_{\mathcal{B}}(\thetastar)=\thetastar$, and since the projection is non-expansive we can bound 
\begin{align}
    ||\thetastar-\thetacl|| =  ||P_{\mathcal{B}}(\thetastar)-P_{\mathcal{B}}(\thetahat)||&\leq ||\thetastar-\thetahat ||,  \label{eq:better_independent_all_t}
\end{align}
 for all  $t=1,\ldots,T$. The desired result \eqref{eq:better_independent} follows from averaging both sides of   \eqref{eq:better_independent_all_t} across $t$. 
 \end{proof}
\begin{figure}[t]
    \centering
    \begin{subfigure}[b]{0.48\textwidth}
        \centering
	    \includegraphics[scale=1]{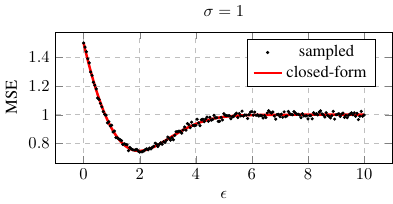}    
    \end{subfigure}
    \hfill
    \begin{subfigure}[b]{0.48\textwidth}
        \centering
	    \includegraphics[scale=1]{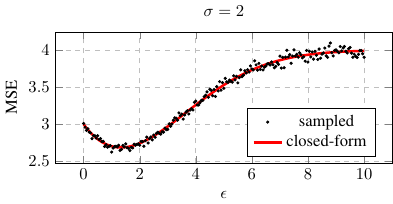}    
    \end{subfigure}
    \caption{Mean square error of the cross-learning estimate as a function of $\epsilon$.  The value of $\epsilon$ interpolates between consensus ($\epsilon=0$), and separate estimation ($\epsilon=\infty$). We present the case with small variance ($\sigma=1$) wherein the separable estimator outperforms the consensus one, and the case with large variance ($\sigma=2$) in which consensus is better than  separate estimation. In either case, the plots exhibit a value of $\epsilon\in(0,\infty)$ for which cross-learning outperforms both the consensus and separable estimators. }
    \label{fig:mse_sigma}
\end{figure}
\noindent\textbf{Example 2\label{example2}} 
To showcase Propositions \ref{prop:centralized_vs_cl} and \ref{prop:independent_vs_cl_estimator}, we present a controlled example in which  data come from a multivariate Gaussian distribution in $\reals^2$. For each task, we consider the centers at $[\mu,0],[-\mu,0], [0,\mu]$, and $[0,-\mu]$. In this case, the mean square error  is $\mu^2+\sigma^2/2$ for the consensus estimator and $2\sigma^2$ for the separable case. Figure \ref{fig:mse_sigma} shows the error $\mathcal E_{CL}(\epsilon)$ found in simulations for $\epsilon\in(0,10)$. There is a value of $\epsilon$ in Figure \ref{fig:mse_sigma} such that the cross-learning  estimator outperforms  both the consensus and separable estimators. This corresponds with  Proposition \ref{prop:centralized_vs_cl}, which proves that there exists a value $\epsilon$ such that the cross-learning estimator outperforms its consensus counterpart. It also lines up with  Proposition \ref{prop:independent_vs_cl_estimator}, which  shows that if $\epsilon\geq \epsilon_0=\max_{t}\max_\tau \|\thetastar-\hat\theta_\tau\|>0$, the mean square error of the cross-learning estimator is bounded from above by the mean square error of the separate counterpart. 

The comparison in Figure \ref{fig:mse_sigma}, however, looks better than our theoretical result in  Proposition \ref{prop:independent_vs_cl_estimator}, which only established a lower-than-or-equal type of bound. Thus, Proposition \ref{prop:independent_vs_cl_estimator}  did not guarantee the performance of the cross-learning estimator to be strictly better than the separate counterpart. To show that indeed the cross-learning estimator outperforms the separate one, we need yet another proposition. In this direction, consider the set 
\begin{align}\mathcal C(\epsilon)\hspace{-2pt}=\hspace{-2pt}\{\hspace{-1pt}(\hat\theta_1,\ldots,\hat\theta_T\hspace{-1pt})\hspace{-2pt}:\hspace{-2pt}\exists t\hspace{-2pt}\in\hspace{-2pt}\{1,\hspace{-1pt}\ldots,\hspace{-1pt}T\}  \hspace{-1pt}\text{ s.t. }\hspace{-1pt} \|\thetahat\hspace{-2pt}-\hspace{-2pt}\thetage\|\hspace{-2pt}\geq\hspace{-2pt} 3 \epsilon\}
\label{eq:set_C}. 
\end{align} 
with $\thetage$ being the centroid in    \eqref{eqn_perf_ana:cross_learning}. Proposition \ref{prop:independent_vs_cl_estimator_strict}, below,
establishes that the cross-learning presents a strictly lower error than the separable estimator when  $(\hat\theta_1,\ldots,\hat\theta_T)$ is chosen from the set   $ \mathcal C(\epsilon)$. Furthermore, the probability of $C(\epsilon)$ is strictly positive, which together with the uniform bound in Proposition \ref{prop:independent_vs_cl_estimator} is sufficient to prove that the mean squared error averaged across all datasets inside and outside $\mathcal C(\epsilon)$ is strictly lower when using cross-learning.   
 
\begin{proposition}\label{prop:independent_vs_cl_estimator_strict}
There exists a value $\epsilon_0>0$ such that for all $\epsilon>\epsilon_0$, if  $\epsilon$ and  $(\hat\theta_1,\ldots,\hat\theta_T)\in\mathcal C(\epsilon)$ are substituted in \eqref{eqn_perf_ana:cross_learning}, the error $\mathcal E_{CL}(\epsilon)$ of the resulting estimator $\thetacl$ is  bounded by   
\begin{align}  
\mathcal E_{CL}(\epsilon)\leq\mathcal E_{S}-\frac{1}{ T}\epsilon^2.\label{eq:strictly_better_independent}
\end{align}
\end{proposition}

\begin{proof}
For each dataset $(\hat\theta_1,\ldots,\hat\theta_T)$ and $\epsilon>0$ compute $\thetage$ via cross-learning, and define 
\begin{align}
\epsilon_0=\inf\{\epsilon>0: \ \|\theta^\star_t-\thetage\|\leq \epsilon  \text{ for all } t=1,\ldots,T \}.
\label{eq:epsilon0_theta}
\end{align}
We have shown in \eqref{eqn:agnostic_condition} that for any $\epsilon\geq \max_{t}\max_\tau \|\thetastar-\hat\theta_\tau\|$ as in \eqref{eq:max_max_dif}, the condition $\|\theta^\star_t-\thetage\|\leq \epsilon$ in  \eqref{eq:epsilon0_theta} is satisfied. Thus, the set in \eqref{eq:epsilon0_theta} is nonempty and the infimum in \eqref{eq:epsilon0_theta} is well defined.  Furthermore,  $\epsilon_0$  is strictly positive, since it depends on the pairwise distances between the ground truth parameters $\theta^\star_t$  which are assumed to be different one each other. Next, we consider the set 
$\mathcal C(\epsilon)$, defined in \eqref{eq:set_C} and depicted in Figure \ref{fig:prueba_mse_projection}, which collects (sufficient statistics of) datasets $(\hat\theta_1,\ldots,\hat\theta_T)$ such that at least one of its points $\hat \theta_t$  belongs to a (blue) zone outside of the ball $\mathcal{B}(\thetage, \epsilon)$. Projecting $\hat \theta_t$  into $\mathcal{B}(\thetage, \epsilon)$ results in a $\thetacl$ that is closer to $\theta^\star_t$. Following this intuition, we will show that \eqref{eq:strictly_better_independent} is satisfied for all  $(\hat\theta_1,\ldots,\hat\theta_T)\in\mathcal C(\epsilon)$. In this direction, we use the triangle inequality to write
\begin{align}
  3\epsilon&\leq   \|\hat\theta_t-\thetage\|=\|\thetahat-\thetastar+\thetastar-\thetage\|\\
    &\leq \|\thetahat-\thetastar\|+\|\thetastar-\thetage\| \leq \|\thetahat-\thetastar\| + \epsilon,
\end{align}
where we used \eqref{eq:epsilon0_theta} and  \eqref{eq:set_C}, valid for one  $\thetahat$. It follows that
\begin{align}
    \|\thetahat-\thetastar\|&\geq  2\epsilon\label{eq:diff_theta_r_star_theta_hat_r}.
\end{align}

On the other hand, since  $\thetacl$ is the projection of $\thetahat$ onto the ball $B(\thetage,\epsilon)$,  then  $\thetahat-\thetacl$ must  be co-linear with $\thetahat-\thetage$ such that  $\thetahat-\thetacl=\mu_t(\thetahat-\thetage)$. Hence, we can write  $\thetahat-\thetage=\thetahat-\thetacl+\thetacl-\thetage=\mu_t(\thetahat-\thetage)+\thetacl-\thetage$ or equivalently  $(1-\mu_t)(\thetahat-\thetage)=\thetacl-\thetage$.  Hence, we can bound 
\begin{align*}
    \|\thetacl-\thetastar\|&=\|\thetacl-\thetage+\thetage-\thetastar\| \\
    &=\|(1-\mu_t) (\thetahat- \thetage) + \thetage- \thetastar\|\\
    &\leq\|\thetahat- \thetastar\| +\mu_r (\|\thetage-\thetastar\|-\|\thetahat- \thetastar\|)\\
    &\leq \|\thetahat- \thetastar\| +\mu_r \left(\epsilon-2\epsilon\right)
    \end{align*}
where we used  \eqref{eq:epsilon0_theta} and \eqref{eq:diff_theta_r_star_theta_hat_r}.  Moreover, since according to \eqref{eq:set_C}  $\thetahat\notin \mathcal{B}(\thetage,\epsilon)$,  the projecting constraint must be active, yielding  $\epsilon=\|\thetacl-\thetage\|=(1-\mu_t)\|\thetahat-\thetage\|\geq 3\epsilon (1-\mu_t) $ according to \eqref{eq:set_C}, which implies  $\mu_t\geq 2/3$.  It follows,
\begin{align}    
    \|\thetacl-\thetastar\|&\leq \|\thetahat- \thetastar\| -\mu_r \epsilon\ \leq \|\thetahat- \thetastar\| -\frac{2}{3}\epsilon\label{eq:cota_mur_set_C}.
\end{align}
 To conclude the proof, we square both sides of   \eqref{eq:cota_mur_set_C} and substitute \eqref{eq:diff_theta_r_star_theta_hat_r} to obtain
\begin{align}
    \|\thetacl-\thetastar\|^2 
    &\leq \left(\|\thetahat- \thetastar\| - \frac{2\epsilon} {3}\right)^2\\
        &\leq \|\thetahat- \thetastar\|^2 -2 \frac{2\epsilon} {3}\|\thetahat- \thetastar\| + \left( \frac{2\epsilon} {3}\right)^2\\
    &\leq \|\thetahat- \thetastar\|^2 -2 \frac{2\epsilon} {3}\left(2\epsilon\right) + \left( \frac{2\epsilon} {3}\right)^2\\
    &\leq \|\thetahat- \thetastar\|^2 -\epsilon^2.\label{eq:strict_bound_r} 
\end{align}
The inequality \eqref{eq:strict_bound_r} is true for at least one $t\in\{1,\ldots,T\}$ such that $\|\thetahat-\thetage\|\geq 3$ as defined in $\mathcal C(\epsilon)$. For all other  $t\in\{1,\ldots,T\}$ we still have \eqref{eq:better_independent_all_t}. Recall that \eqref{eq:better_independent_all_t} relies on   $\|\theta^\star_t-\thetage\|\leq \epsilon$ in \eqref{eqn:agnostic_condition}, which holds for all $\epsilon>\epsilon_0$. Thus, we can average across $t$ and the slack $\epsilon^2$ in \eqref{eq:strict_bound_r} will reduce by a factor $T$ which results in   \eqref{eq:strictly_better_independent}.
\end{proof}
If we can show that $\mathcal C(\epsilon)$ has positive probability, then the cross-learning estimator will have a strictly lower mean-square error than the separable one. 
We will follow that path in establishing that the cross-learning estimator outperforms the separate and consensus ones in our main result.  
\begin{theorem}\label{theorem:lower_mse}
    Consider a set of deterministic parameters $\thetastar$,  $t=1,\ldots,T$  and their associated datasets collecting data corrupted by zero-mean  Gaussian noise $\nu_{tn}\sim  \mathcal N(0,\sigma)$ according to the additive model $ y_{tn}= \thetastar+\nu_{tn},\ t=1,\ldots,T,\ n=1,\ldots, N_t$.   Define the  square errors $\mathcal E_{S}$, $\mathcal E_{C}$, and $\mathcal E_{CL}(\epsilon)$, as in \eqref{eqn_mseI_def}, \eqref{eqn_mseC_def}, and \eqref{eqn_mseCL_def} by quantifying the error between a single ground truth parameter $\thetastar$ and the solutions of the separable \eqref{eqn:hatP_indepedent}, consensus \eqref{eqn:hatP_centralized}, and cross-learning \eqref{eqn_perf_ana:cross_learning} estimators, respectively. Under these definitions, we have
      \begin{align}
&\mathbb{E}\left[\inf_{\epsilon>0} \mathcal E_{CL}(\epsilon)-\mathcal E_{C}\right]<0,\label{eqn:better_than_C}\\
&\mathbb{E}\left[\inf_{\epsilon>0} \mathcal E_{CL}(\epsilon)-\mathcal E_{S}\right]< 0.\label{eqn:better_than_I}
\end{align}
\end{theorem}
\begin{proof}
We start by comparing the errors yield by consensus and by the separate estimator, dividing the proof in two complementary cases $\mathbb{E}\left[\mathcal E_{C}\right]\leq\mathbb{E}\left[\mathcal E_{S}\right]$ or $\mathbb{E}\left[\mathcal E_{C}\right]> \mathbb{E}\left[\mathcal E_{S}\right]$. Starting with the case $\mathbb{E}\left[\mathcal E_{C}\right]<\mathbb{E}\left[\mathcal E_{S}\right]$, we established in Proposition \ref{prop:centralized_vs_cl} that there exists an $\epsilon>0$ such that \eqref{eqn:better_than_C} holds. Since we assumed $\mathbb{E}\left[\mathcal E_{C}\right]<\mathbb{E}\left[\mathcal E_{S}\right]$, then \eqref{eqn:better_than_I} must also hold. 
In the case  $\mathbb{E}\left[\mathcal E_{S}\right]<\mathbb{E}\left[\mathcal E_{C}\right]$, we will first argue that for any dataset, we can find a value of $\epsilon$ such that $\mathcal E_{CL(\epsilon)}- \mathcal E_S$  is lower than or equal to zero. But, this is a direct consequence of Proposition \ref{prop:independent_vs_cl_estimator}, since the right-hand side of \eqref{eq:max_max_dif} only depends on the dataset. Therefore, \eqref{eqn:better_than_I} holds, and since we are considering the case $\mathbb{E}\left[\mathcal E_{S}\right]<\mathbb{E}\left[\mathcal E_{C}\right]$, then \eqref{eqn:better_than_C} must hold. To ensure that the inequality in  \eqref{eqn:better_than_I} is strict we argue that the set $\mathcal C(\epsilon)$ in \eqref{eq:set_C} is nonempty because we can separate the points $\hat\theta_t$ away in Fig. \ref{fig:prueba_mse_projection} keeping $\thetage$ unchanged, so that $\|\thetahat-\thetage\|$ surpasses $3\epsilon$ in \eqref{eq:set_C}.  Furthermore, if we allow $\thetage$ to move infinitesimally, we can move the points $\thetahat$ in a set $\mathcal C(\epsilon)$ with positive volume, so that the probability of $\mathcal C(\epsilon)$ is positive under our Gaussian data model. This positive probability together with \eqref{eq:strictly_better_independent} yields \eqref{eq:better_independent}, concluding the proof.     
\end{proof}

The strict inequalities in Theorem \ref{theorem:lower_mse}  mean that by selecting the value of $\epsilon$ properly (possibly depending on the dataset samples $\{y_{tn}\}$) the cross-learning estimator outperforms its separable and consensus counterparts. Through the supporting Propositions \ref{prop:centralized_vs_cl} and \ref{prop:independent_vs_cl_estimator}, which are described by Figures \ref{fig:proof_symmetry} and \ref{fig:prueba_mse_projection} respectively, we recognize that the cross-learning constraint can simultaneously reduce the bias affecting the consensus estimator, and mitigate the higher variance of a separate result, attaining an optimal balance in between.  Although these theoretical findings were formally established within a controlled framework with Gaussian data, this intuition about how the cross-learning method achieves this optimal balance in the variance-bias tradeoff generalizes to more complex scenarios. Specifically, this balance carries out to  the  experiments of  Section \ref{subsec:NumericalCovid} which are performed on real data. Before presenting these experiments, we propose a variant of cross-learning that is preferable when the model outputs, as opposed to their parameters, are assumed to be close to one another.
\section{Cross-Learning with Coupled Outputs}
\label{sec:functional}

Our cross-learning formulation \eqref{prob:CL} so far has put emphasis on the model parameters. While there are many problems where obtaining the parameters of a model is the primary goal, in most cases, they are just a vehicle to obtain a regression or classification output. Particularly when neural networks are involved, the optimal values of the filter taps provide little to no intuition into the phenomenon being learned. Moreover, close distances between these parameters may not reflect similar outputs to the same inputs.  In these cases, it may be reasonable to leverage similarities in the model  outputs themselves, as opposed to their parameters, for which we introduce an alternative cross-learning formulation with functional constraints
\begin{align}
	\{\thetacl\},\thetage=\underset{\{\theta_t\},\theta_g}{\arg\min}  \quad   &  \frac{1}{T}\sum_{t=1}^T\frac{1}{N_t}\sum_{i=1}^{N_t} \ell \left(y_i,f(x_i,\theta_t)\right)] \label{prob:CL_functional}  \tag{$\text{P}_{\text{CLF}}$}
	\\
	\text{subject to}
    \quad   &\frac{1}{N_t}\sum_{i=1}^{N_t} | f(x_i,\theta_t)-f(x_i,\theta_g)| \leq \epsilon,\nonumber
\end{align}
with $\ell:\reals ^Q\times\reals^Q\to\reals_+,\ \ell(y_1,y_2)=0 $ iif $y_1=y_2$, being a non-negative loss function that measures the quality of learning. As before, we fit the parameters $\theta_t$ of the models $f(x,\theta_t)$ jointly across tasks, but now the model outputs are coupled via the functional constraints.

While we will not provide a formal analysis of this case, we claim that this cross-learning formulation \eqref{prob:CL_functional} with functional constraints comes to also exploit the bias-variance trade-off inherent to the problem. We support this claim with the experiments of Section \ref{ssec:image_Classification}, and by relating \eqref{prob:CL_functional} to the parametric case \eqref{prob:CL}. Specifically, we notice that  \eqref{prob:CL_functional} is a relaxation of \eqref{prob:CL}, as stated by the following proposition.
\begin{proposition}\label{prop:relaxationset}
If the model $f(x,\theta)$ is $L$-Lipschitz in the parametrization, i.e., $| f(x,\theta_t)-f(x,\theta_g)| \leq L |\theta_t - \theta_g|$, then the feasible sets $\mathcal{C}_{CL}(\epsilon)$ and $\mathcal{C}_{CLF}(\epsilon)$ for the cross-learning estimators \eqref{prob:CL} and \eqref{prob:CL_functional} with coupled parameters and constraints, respectively, satisfy $\mathcal{C}_{CL}(\epsilon) \subseteq \mathcal{C}_{CLF}(L\epsilon)$. 
\end{proposition}

\begin{proof}
Let $(\theta_t, \theta_g) \in \mathcal{C}_{CL}(\epsilon)$ such that
\begin{align}
    \mathcal{C}_{CL}(\epsilon)&=\left\{ (\theta_t, \theta_g) : \|\theta_t - \theta_g\| \leq \epsilon \right\}.\label{eq:CLconstraintset}
\end{align}
hence,  $\|\theta_t - \theta_g\| \leq \epsilon$ holds. Furthermore, by application of the Lipschitz property of the model function $f$, we have that
\begin{align}
\|f(x_i, \theta_t) - f(x_i, \theta_g)\| \leq \| \theta_t - \theta_g \| \leq L \epsilon
\end{align}
Averaging over the samples $i=1,\ldots,N_t$ we obtain the bound 
$\frac{1}{N_t}\sum_{i=1}^{N_t}\|f(x_i, \theta_t) - f(x_i, \theta_g)\| \leq \frac{1}{N_t}\sum_{i=1}^{N_t} L \epsilon = L \epsilon$.
This implies $(\theta_t, \theta_g)\in  \mathcal{C}_{CLF}(L\epsilon)$ since according to \eqref{prob:CL_functional}, 
  $\mathcal{C}_{CLF}(\epsilon)\hspace{-2pt}=\hspace{-2pt}\{(\theta_t, \theta_g)\hspace{-2pt}:\frac{1}{N_t}\sum_{i=1}^{N_t}|f(x_i,\theta_t)-f(x_i,\theta_g)|\hspace{-2pt}\leq\hspace{-2pt}\epsilon\}$. 
\end{proof}

The inclusion $\mathcal{C}_{CL}(\epsilon) \subseteq \mathcal{C}_{CLF}(L\epsilon)$ means that \eqref{prob:CL_functional} imposes a \emph{weaker} coupling between task-specific parameters and the shared representation compared to the parametric problem. Consequently, while both formulations promote cross-task information sharing, the output constraint does so in a more flexible manner. In other words, \eqref{prob:CL_functional} allows for greater variability among task-specific parameters as long as the corresponding function outputs remain sufficiently close.
In this case, the solution of \eqref{prob:CL_functional} with $\epsilon=0$ does not imply strict model consensus, since the parameters could differ across tasks and the outputs could also differ for inputs not seen during training. As $N_t$ grows and the training set becomes more representative,  the models reach consensus, which introduces bias as described in previous sections. If $N_t$ reduces, we still expect a residual bias under this weaker broad-sense consensus enforced by \eqref{prob:CL_functional} with $\epsilon=0$. On the other hand, when $\epsilon\to \infty$ the conclusion is more direct. That is, we recover the fully separable case and its higher variance as we did when coupling the parameters. As in the parametric case, the cross-learning estimator, in this new form with coupled outputs, sets itself in between these consensus and separable counterparts.

\section{Dual Algorithms}
\label{sec:DualDomain}
Upon presenting the cross-learning estimators \eqref{prob:CL} and \eqref{prob:CL_functional} as formulations to better balance the variance-bias trade-offs, in this section we propose two methods to solve their corresponding optimization problems.

\subsection{Dual Algorithm for Parametric Constraints}

Our first algorithm is built using the Alternating Direction Method of Multipliers (ADMM) \cite{boyd2011distributed}. This formulation is suitable to solve the optimization problem \eqref{prob:CL} that defines the  cross-learning estimator with parametric constraints. In the case in which an algorithm is available to solve the separable problem \eqref{prob:MTL_independent_EMP}, this ADMM formulation allows us to solve  the cross-learning problem \eqref{prob:CL} using the separable solution as a building block. We want to obtain the optimal cross-learning parameters  $\thetacl$ for each of the tasks $t=1,\ldots,T$. To reduce notation, let us define the  $\theta = (\theta_1,\ldots,\theta_T,\theta_g)$ containing all optimization variables and write the loss in terms of $\theta$ and the data $(x_t,y_t)$ for task $t$  as 
\begin{align}
\ell(\theta)&=\sum_t \ell\left(y_t, f(x_t,\theta_t)\right)=\sum_t ||y_t-f(x_t,\theta_t)||^2
\end{align}
Under this notation, the cross-learning problem to solve is 
\begin{align}\min_{\theta} &\ \ell(\theta),\ 
\text{subject to}\ \theta \in \mathcal C \label{eq:probCL_ADMM}
\end{align}
with $\mathcal C=\feasible$ being the feasible set defined in \eqref{eq:CLconstraintset}.
By introducing the barrier function 
\begin{align}C(\theta)= \begin{cases}0 & \theta \in \mathcal C \\ \infty & \theta \notin \mathcal C\end{cases}\end{align}
together with an auxiliary variable $z$, we can write \eqref{eq:probCL_ADMM} as 
\begin{align}\min_{\theta ,z} &\ \ell(\theta)+C(z),\ \label{eq:consensus_cl}
\text{subject to } \theta=z 
\end{align}
Written as \eqref{eq:consensus_cl}, we can solve cross-learning  via the ADMM
\begin{align}\theta^{k+1}&=\operatorname{Prox}_{\lambda_k \ell}\left(z^k-u^k \right)\label{eq:proxloss}\\
z^{k+1}&= P_{\mathcal C}\left(\theta^{k+1}+u^k\right)\label{eq:admm_z} \\
u^{k+1}&=u^k+{v^{k+1}-z^{k+1}}\end{align}
with  the proximal operator of the loss in \eqref{eq:proxloss}  defined by 
\begin{align}
\theta^{k+1}\hspace{-2pt}=\hspace{-1pt}\underset{\theta}{\arg\min}
&\hspace{-2pt}\sum_t\hspace{-2pt} \ell(y_t,f(x_t,\theta_t))
\hspace{-2pt}+\hspace{-2pt}\frac{1}{2 \lambda_k}\|\theta\hspace{-1pt}-\hspace{-2pt}z^k\hspace{-2pt}+\hspace{-2pt}u^k\|^2\label{eq:proximal_norm}
\end{align}
where $\lambda_k=\lambda_0 \Gamma^{k}$ is a weighting parameter that progressively accentuates the effect of the loss over the  quadratic  penalty.
  
\begin{algorithm}[t]
	\caption{Cross-Learning with Parametric Constraints}
	\label{alg:AlgorithmParametric}
	\begin{algorithmic}[1]
		\State Initialize: $\theta^0, z^0=0,$ and $u^0=0$ 
		\For{$k = 0, 1, 2, \ldots , k_{max}$}
		    \State Compute $\lambda_k = \lambda_0 \Gamma^{k \bmod K}$
			\For{each task $t = 1,\dots,T$}
				\State $\theta_t^{k+1} = \arg\min_{\theta} \ell_{\lambda_k}(y_t,f(x_t,\theta),z_t,u_t) $
			\EndFor
			\State $\theta_g^{k+1} = z_g^k - u_g^k$
			\State $z^{k+1} = P_{\mathcal{C}}(\theta^{k+1} + u^k)$
			\State $u^{k+1} = u^k + \theta^{k+1} - z^{k+1}$
		\EndFor
	\end{algorithmic}
\end{algorithm}
 
To take full advantage of the ADMM, we expand the penalty  
$\|\theta-z^k+u^k\|^2=\sum_{t=1}^T \|\theta_t-z_t^k+u_t^k\|^2+\|\theta_g-z_g^k+u_g^k\|^2$ 
with $z_t^k$ and  $u_t^k$  denoting the elements of $z^k$ and $u^k$, respectively,  corresponding to task $t$, and  $z_g^k$ and  $u_g^k$ being the ones corresponding to the centroid. Such an expansion facilitates a distributed solution of  \eqref{eq:proximal_norm} in which the parameters $\theta_t$ for task $t$  are obtained by solving      
\begin{align}
\theta^{k+1}_{t}&=\underset{\theta}{\arg\min} \ \ell_{reg}(y_t,f(x_t,\theta_t),z_t,u_t) \label{eq:ellreg}
\end{align}
with the  regularized loss   defined by 
$\ell_{\lambda_k}(y_t,f(x_t,\theta_t),z_t,u_t)=\ell(y_t,f(x_t,\theta_t))
+\frac{1}{2 \lambda_k}\|\theta_{t}-z_{t}^k+u_{t}^k\|^2$.

The centroid, in turns, is only involved in the penalty in \eqref{eq:proximal_norm}. Hence, it admits the  closed-form $\theta_g=z_{g}-u_{g}$. This describes the entire procedure, which is  summarized in Algorithm \ref{alg:AlgorithmParametric}.

\subsection{Dual Algorithm for Coupled Outputs}
Next, we construct a primal-dual algorithm   to solve \eqref{prob:CL_functional}. With $\lambda_t>0$ denoting the dual variable associated with domain $t$, and $ \lambda=[\lambda_1,\dots,\lambda_T]^T$, the Lagrangian associated with  \eqref{prob:CL_functional} takes the form
\begin{align}
	 L(\theta_t,\theta_g,\lambda)&= \frac{1}{T}\sum_{t=1}^T \frac{1}{N_t} \sum_{i=1}^{N_t} \ell \left(y_i,f(x_i,\theta_t)\right) \nonumber\\ &+\lambda_t\bigg(\frac{1}{N_t} \sum_{i=1}^{N_t}| f(x_i,\theta_t)-f(x_i,\theta_g)|- \epsilon\bigg).
\end{align}
We define the dual function for problem \eqref{prob:CL_functional}  as 
	$d(\lambda) := \min_{\theta_t,\theta_g} L (\theta_t,\theta_g,\lambda)$, and the corresponding convex dual problem as the maximization of $d(\lambda)$ in the positive orthant, that is, $D^\star_{CLF} :=\max_{\{\lambda_t\geq 0\}} d(\lambda)$ \cite{boyd2009convex}.

Problem $D_{CLF}^\star$ can be solved \cite{27Chamon2020,28Chamon2020} by alternating gradient steps on the Lagrangian and the dual function. Upon selecting step-sizes $\eta_P>0$ and $\eta_D>0$, and initializing the parameters $\{\theta_t^0\},\theta_g^0$, and the dual variables $\lambda_t^0$, we  update the parameters by taking a gradient descent step on $L(\theta_t,\theta_g,\lambda)$,
\begin{align}
\theta^{k+1}_t &= \theta^k_t -\eta_P \nabla_{\theta_t}  L(\theta_t,\theta_g,\lambda)\text{ for all } t\in[1,\dots,T],\\
\theta^{k+1}_g &= \theta^t_g -\eta_P \nabla_{\theta_g} L(\theta_t,\theta_g,\lambda),
\end{align}
and then  the multipliers with a gradient ascent step over  $d(\lambda)$ 
\begin{align*}
\lambda^{k+1}_t = \bigg[\lambda^{k}_t + \eta_D\bigg(\frac{1}{N_t} \sum_{i=1}^{N_t}| f(x_i,\theta_t)-f(x_i,\theta_g)|- \epsilon \bigg)\bigg]_+
\end{align*}
for $t=1,\ldots T$, projecting the result into the nonnegative orthant via $[\cdot]_+=\max\{0,\cdot\}$. The overall  is detailed in Algorithm \ref{alg:Algorithm}. We refer the reader to \cite{26Chamon2022} for a proof of convergence in a stochastic setup.

\begin{algorithm}[t]
	\caption{Cross-Learning Functional Algorithm}
	\label{alg:Algorithm}
	\begin{algorithmic}[1]
		\State Initialize models $\{\theta^0_t\},\theta^0_g$, and dual variables $\lambda = 0$
		\For {epochs $e=1,2,\dots$}
		\For {batch $i$ in epoch $e$}
		\State Update params. $\theta_t^{k+1}=\theta_t^{k}-\eta_P \hat\nabla_{\theta_t} L (\theta_t,\theta_g,\lambda)\forall \ t$
		\State Update $g$ params. $\theta_g^{k+1}=\theta_g^{k}-\eta_P \hat\nabla_{\theta_g} L (\theta_t,\theta_g,\lambda)$
		\EndFor
		\State Update dual variable for all $t\in[1,\dots,T]$\\ 
    $\lambda_t^{k+1} = \biggl[\lambda^{k}_t + \eta_D\bigg(\frac{1}{N_t} \sum_{i=1}^{N_t}| f(x_i,\theta_t)-f(x_i,\theta_g)|- \epsilon \bigg)\biggl]_{+}$ 
		\EndFor
	\end{algorithmic}
\end{algorithm}

\section{Numerical Examples}

In this section, we test our cross-learning estimators on real-world data in two scenarios: a time-series prediction task using COVID-19 epidemiological data \cite{owid-coronavirus} and an image classification task using the Office-Home dataset \cite{venkateswara2017deep}.

\subsection{COVID-19 SIR Model Fitting}\label{subsec:NumericalCovid}
In this section, we formulate an SIR fitting problem for the COVID-19 pandemic data obtained from the OWID database \cite{owid-coronavirus}. Although complex models are often required to capture the full dynamics of an epidemic, we use a simple, interpretable, bi-parametric SIR model. This choice will allow for a clear interpretation of our cross-learning method. The SIR model is described by a system of  differential equations 
\begin{align}
\frac{dS}{d\tau}=-\frac{\beta}{N} S I, \quad
\frac{dI}{d\tau}=\frac{\beta}{N} S I - \gamma I, \quad
\frac{dR}{d\tau}= \gamma I,\label{eq:sir_model}
\end{align}
where variables $S$, $I$, and $R$ stand for the number of susceptible, infected, and removed individuals, respectively, which evolve with time $\tau$, and the constant $N=S+I+R$ stands for the initial total population.

We consider the multiple tasks of predicting the infection time series $(S_t,I_t,R_t)$ for $T$ different countries. 
The loss $\ell$ is defined as the mean squared error between the predicted and observed infection data for each country. 
By following Algorithm \ref{alg:AlgorithmParametric}, which optimizes the regularized losses $\ell_{\lambda_k}$, we obtain specific parameters $\thetacl = (\beta_t, \gamma_t)$ for each country, together with a  centroid $\thetage$.
\begin{figure}
    \centering
	\includegraphics[scale=1]{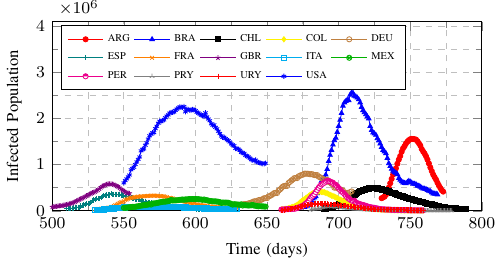}        
    \caption{Infected population over time for $T=14$ different countries in the OWID dataset. Countries exhibit different dynamics for the evolution of their infected populations.}
    \label{fig:covid-waves}
\end{figure}
The central question we investigate is: Would it have been possible to predict the peak and timing of Argentina's (ARG) COVID-19 wave using only early data from their own incipient records, supplemented with data from previous waves in other countries? To answer this, we use the epidemic waves from $T-1=13$ countries, shown in Figure 4. Our goal is to predict the latter part of the ARG wave using only its initial data points. We evaluate performance using two metrics: the lag error, which measures the difference in days between the predicted and actual peak of the infection wave, and the peak error, which measures the difference in the number of infected people at the peak.
Figure \ref{fig:10points} illustrates the prediction results when using only the first 10 days of ARG wave data and all previously available data from Fig. \ref{fig:covid-waves}. The separate estimator, which uses only ARG data, fails to capture the peak. With a fitted recovery rate $\gamma = 0$, it incorrectly predicts unbounded exponential growth. On the other hand, enforcing consensus, yields a single common pair of $(\beta,\gamma)$ parameters for all $T=14$ countries, averaging out the unique dynamics of the ARG wave, resulting in a prediction that severely underestimates the severity of the outbreak. In contrast, the cross-learning approach (with $\epsilon=0.1$) effectively leverages the data from all 15 countries. It produces a model that accurately predicts both the timing (lag) and the magnitude (peak) of the ARG infection wave.

 Under cross-learning, the distance of the final parameters in parameter space depends on the decision of the value $\epsilon$. Particularly $\epsilon=\infty$ forces no closeness between parameters, being equivalent to a separate approach, while $\epsilon=0$ forces all parameters to be equal, resulting in strict consensus. Figure \ref{fig:parameterspace} compares the separate SIR estimators to the cross-learning ones for $\epsilon=0.1$. A separate estimator for ARG yields $\gamma=0$, which is expected since the parameter $\gamma$, specifically, is tied to the  downward slope, and there are not enough points in the ARG dataset to fit it properly. By contrast, the parameters $\thetacl=(\beta_t,\gamma_t)$ are closer to the centroid  $\thetage$. Thus, the cross-learning estimator forces both ARG parameters to be nonzero. 
 \begin{figure}
    \centering
	\includegraphics[scale=1]{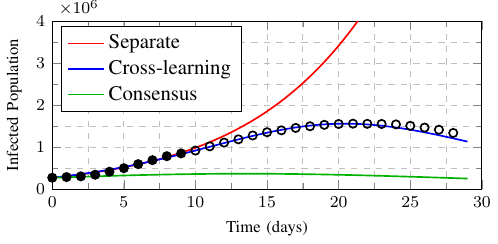}        
    \caption{SIR model predictions for ARG using the separate $(\beta=0.1481,\gamma=0.0)$, cross-learning $(\beta=0.6608,\gamma=0.4388)$, and consensus $(\beta=0.3497,\gamma=0.2802)$ estimators. Filled and hollow black dots represent the training and test datasets, respectively. Cross-learning achieves a more accurate prediction of the peak of infections with an error of $0.07\%$ in the number of cases and finding the exact day when the peak occurs, as compared to errors of $2474\%$ and $76.38\%$ and time lags of $108$ and $8$ days for the separate and consensus estimators, respectively.}
    \label{fig:10points}
\end{figure}
 Close inspection of the SIR model \eqref{eq:sir_model} reveals that the peak of infections occurs when $\beta_t S_t/N_t=\gamma_t$, since that implies $dI_t/d\tau=0$. Furthermore, $\beta_t$ and $\gamma_t$ must be similar to each other if the peak occurs at an early stage of the epidemic, in which $S_t/ N_t\simeq 1$. As Figure  \ref{fig:parameterspace} shows, this similarity is captured by the cross-learning estimator and is a key enabler for predicting the peak in Figure \ref{fig:10points}, where the separate estimator fails. Notice also that, while the centralized estimator also forces $\beta_c\simeq \gamma_c$, it reduces the influence of ARG data on the estimation. This causes distortion on the parameter $\beta$, which models the initial upright slope. As seen in Figure \ref{fig:10points}, such a distortion results in a worse prediction of the peak compared the one obtained by our cross-learning approach.  Remarkably, cross-learning  does not require specific knowledge of these model insights but captures them by just setting a rather simple  similarity constraint for the parameters across tasks.   
Furthermore, we performed an ablation study whose results are shown in   Figure \ref{fig:ablation-epsilon}. Choosing $\epsilon=0.1$ is optimal, in the sense that it yields both the minimum peak and lag errors. Still, there is a wide range of $\epsilon$ values yielding errors that are orders of magnitude lower than those corresponding to  the separate or consensus estimators, indicating that the approach is robust and does not require extensive hyperparameter tuning to provide significant benefits.
\begin{figure}
    \centering
	\includegraphics[scale=1]{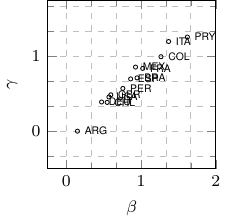}    
	\includegraphics[scale=1]{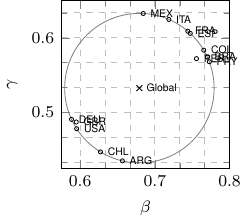}        
    \caption{Learned SIR parameters $(\beta,\gamma)$ per country by separate estimators (left) and cross-learning  (right).}
    \label{fig:parameterspace}
\end{figure}
\begin{figure}
    \centering
	\includegraphics[scale=1]{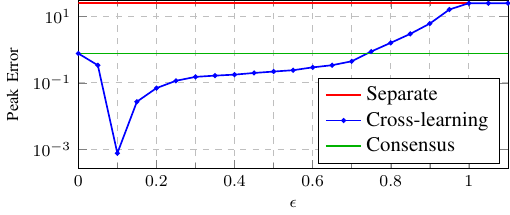}    
	\includegraphics[scale=1]{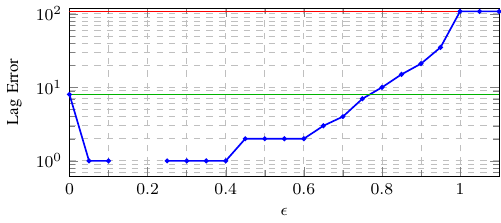}                
    \caption{Study on the effect of the centrality parameter $\epsilon$ on the peak error (top) and lag error (bottom) in logarithmic scale. The cross-learning estimator coincides with the consensus one at $\epsilon=0$, reduces to a minimum error (with no lag for $\epsilon \in \{0.15,0.20\}$.) and then reaches its assymptotic value, the error corresponding to the separate estimator at $\epsilon= 1.0$.}
    \label{fig:ablation-epsilon}
\end{figure}

Finally, we analyze the prediction error as a function of the number of available data points for ARG. In Figure 8, we show the error using separate, consensus, and cross-learning approaches. The independent approach requires a substantial amount of data (nearly 20 days) to correctly identify the trend and make an accurate prediction. In contrast, our cross-learning method achieves a low error with as few as $10$ days of data. The performance of the consensus approach remains poor regardless of the amount of ARG data, as this specific information is diluted within the global dataset.
\begin{figure}
    \centering
	\includegraphics[scale=1]{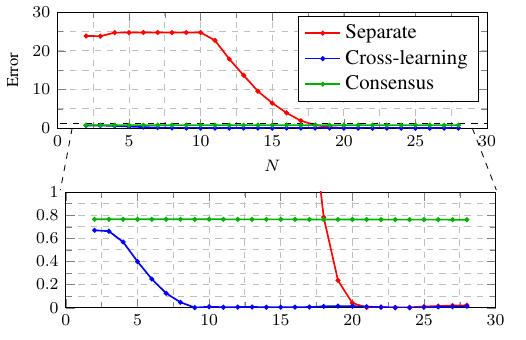}                    
   \vspace{-3ex}
    \caption{Prediction error as a function of the number of available ARG data points ($N$). The bottom panel provides a zoomed-in view of the low-error region.}
    \label{fig:peaksvsn}
\end{figure}

\subsection{Office-Home Dataset Classification}\label{ssec:image_Classification}
In this subsection, we benchmark our cross-learning method with coupled outputs on an image classification problem with real data coming from the dataset that we introduced in Figure \ref{fig:Dataset}. We consider the problem of classifying images belonging to $P=65$ different categories (which in Figure \ref{fig:Dataset} are Alarm, Bike, Glasses, Pen, and Speaker), and $T=4$  different domains. The Office-Home dataset \cite{venkateswara2017deep} consists of $15{,}500$ RGB images in total coming from the domains (i.e. tasks) (i) Art: an artistic representation of the object, (ii) Clipart: a clip art reproduction, (iii) Product: an image of a product for sale, and (iv) Real World: pictures of the object captured with a camera. Intuitively, by looking at the images, we can conclude that the domains are related. The minimum number of images per domain and category is $15$ and the image size varies from the smallest image size of $18 \times 18$ to the largest being $6500 \times 4900$ pixels. We pre-processed the images by normalizing them and fitting their size to $224 \times 224$ pixels.

 We  classify these images with neural networks $f(x,\theta_t)$,  with their architecture based on AlexNet \cite{NIPS2012_c399862d}, reducing the size of the last fully connected layer to $256$ neurons. We do not  pre-train these  networks, but optimize their weights from data via cross-learning using the cross-entropy loss \cite{1Hastie2009}, after splitting the dataset in  $4/5$ of the images for training and $1/5$ for testing.    
Since neural networks are involved, we opt for the version of cross-learning with coupled outputs. Specifically, we use Algorithm \ref{alg:Algorithm} to train $T+1=5$ neural networks for the $T=4$ domains plus the centroid, with step-sizes $\eta_P=0.003$ and $\eta_D=10$, respectively, for different values of $\epsilon$. 
We also used the same split dataset to  test the consensus classifier, which is equivalent to merging the images from all domains and training a single neural network on the whole dataset, and the $T=4$  classifiers which train their  neural networks separately from images of their own specific domains. For comparison purposes, we also run Algorithm \ref{alg:AlgorithmParametric} on this dataset, which implements cross-learning with  coupled parameters.

\begin{figure}[t]
\captionsetup[subfigure]{labelformat=empty}
\centering
\begin{subfigure}[b]{.04\columnwidth}
\rotatebox{90}{\quad \small{Art}}
\end{subfigure}
\begin{subfigure}[b]{.13\columnwidth}
\includegraphics[width=1.2cm,height=1.2cm,keepaspectratio]{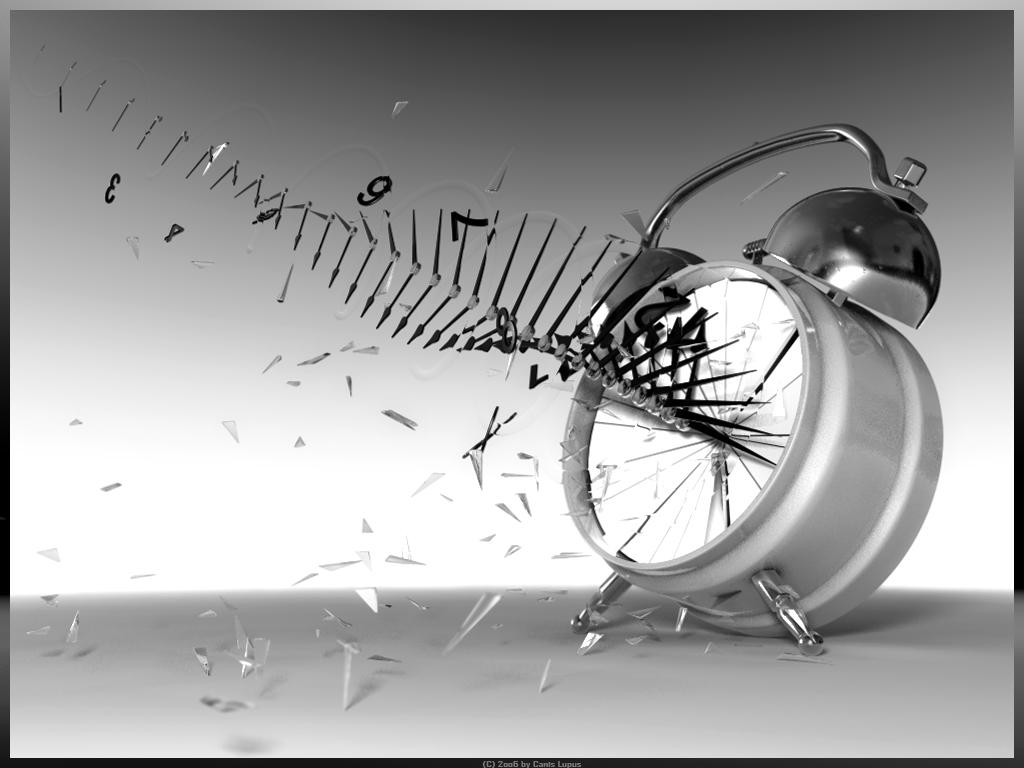}   
\end{subfigure}
\begin{subfigure}[b]{.13\columnwidth}
\includegraphics[width=1.2cm,height=1.2cm,keepaspectratio]{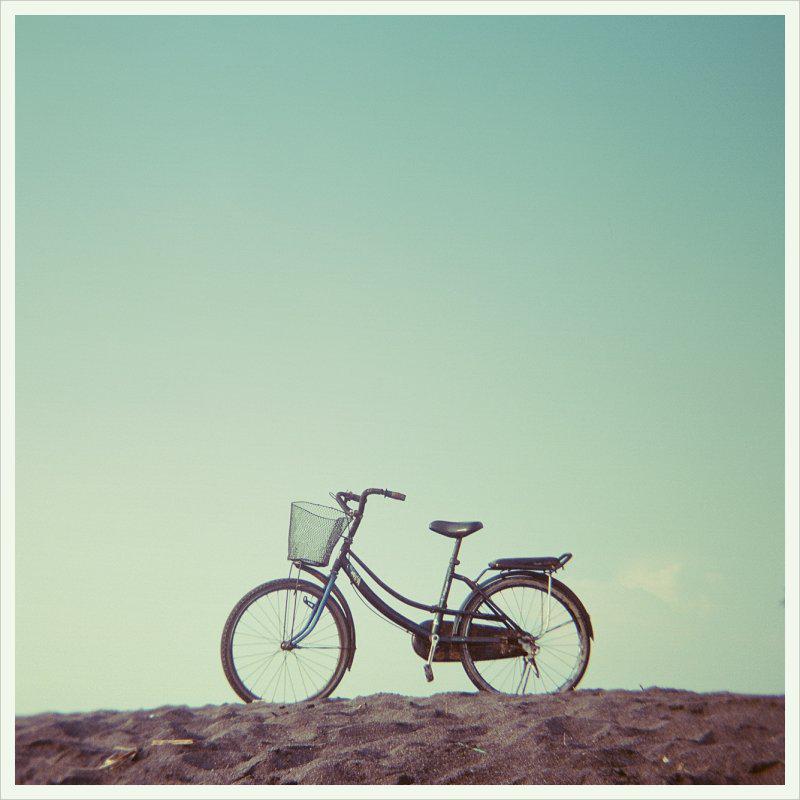}      
\end{subfigure}
\begin{subfigure}[b]{.13\columnwidth}
\includegraphics[width=1.2cm,height=1.2cm,keepaspectratio]{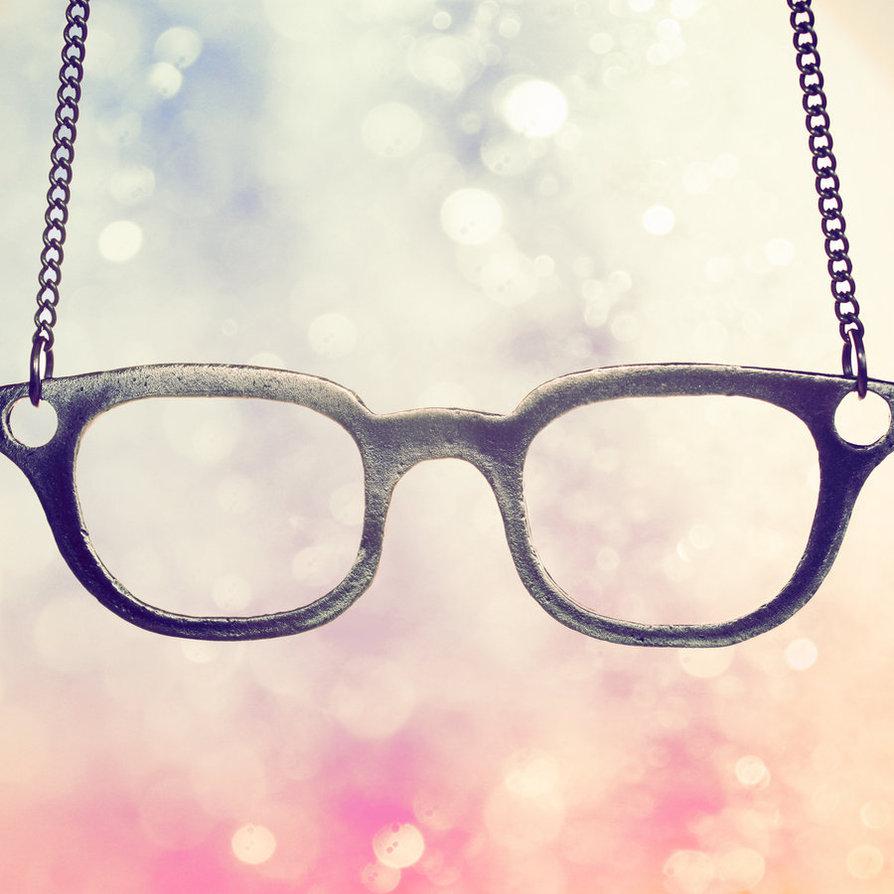}    
\end{subfigure}
\begin{subfigure}[b]{.13\columnwidth}
\includegraphics[width=1.2cm,height=1.2cm,keepaspectratio]{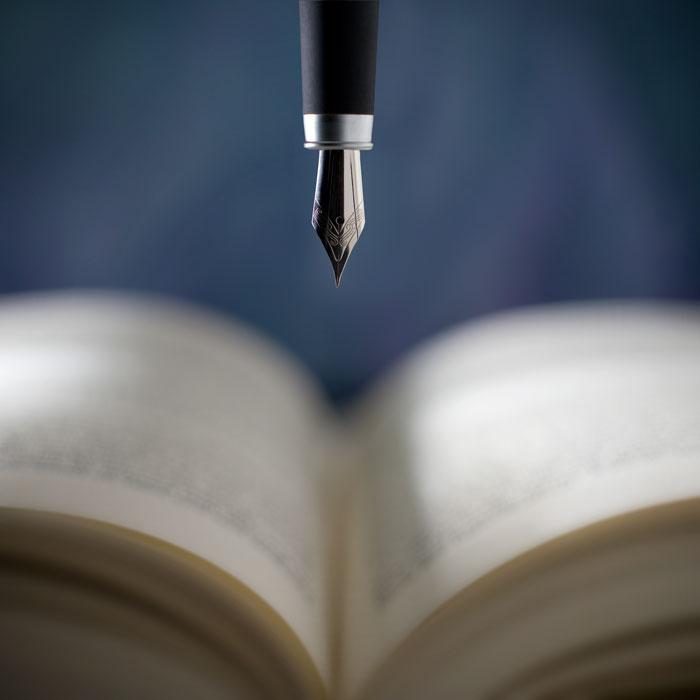}      
\end{subfigure}
\begin{subfigure}[b]{.13\columnwidth}
\includegraphics[width=1.2cm,height=1.2cm,keepaspectratio]{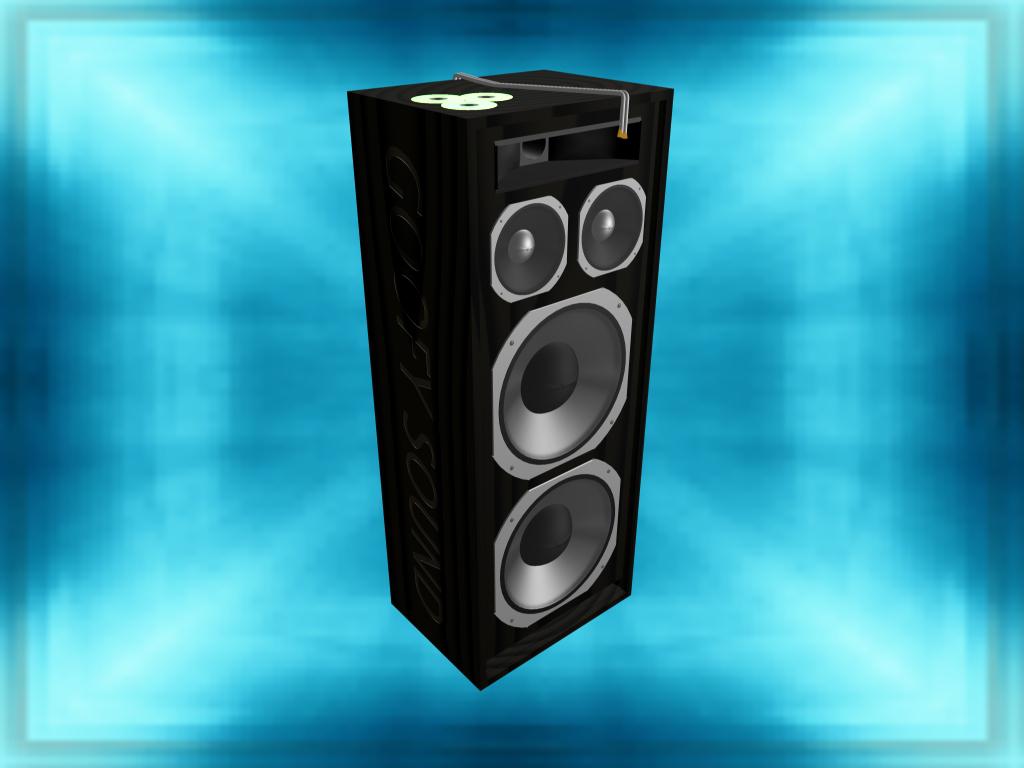}      
\end{subfigure}
\\
\begin{subfigure}[b]{.04\columnwidth}
\rotatebox{90}{\enskip \small{Clipart}}
\end{subfigure}
\begin{subfigure}[b]{.13\columnwidth}
\includegraphics[width=1.2cm,height=1.2cm,keepaspectratio]{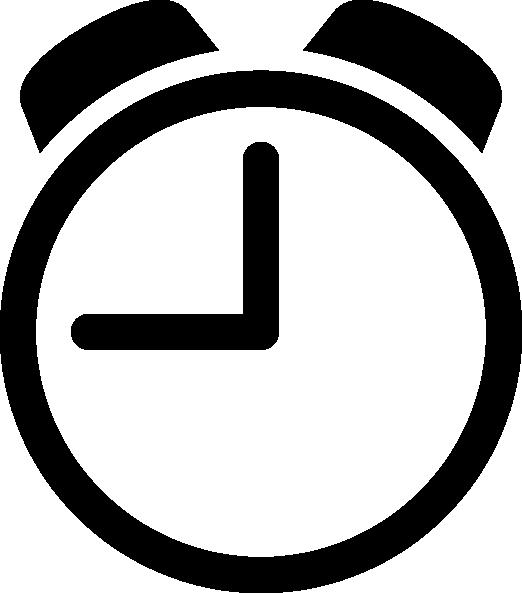}   
\end{subfigure}
\begin{subfigure}[b]{.13\columnwidth}
\includegraphics[width=1.2cm,height=1.2cm,keepaspectratio]{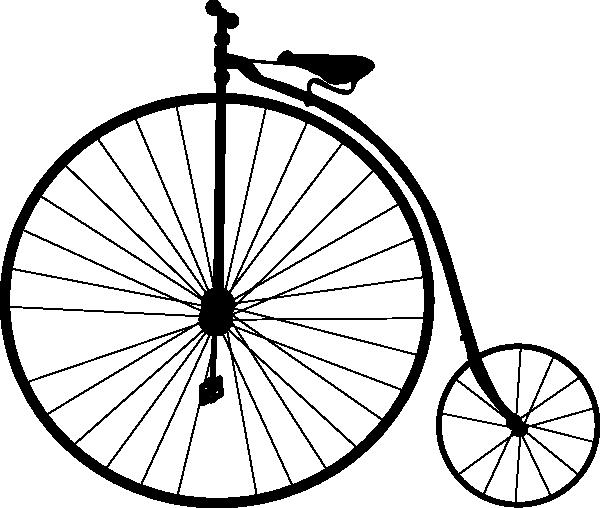}      
\end{subfigure}
\begin{subfigure}[b]{.13\columnwidth}
\includegraphics[width=1.2cm,height=1.2cm,keepaspectratio]{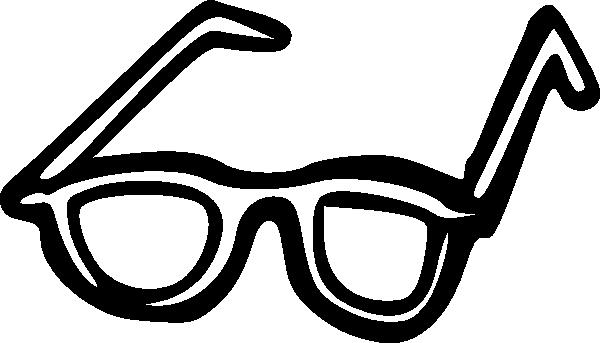}    
\end{subfigure}
\begin{subfigure}[b]{.13\columnwidth}
\includegraphics[width=1.2cm,height=1.2cm,keepaspectratio]{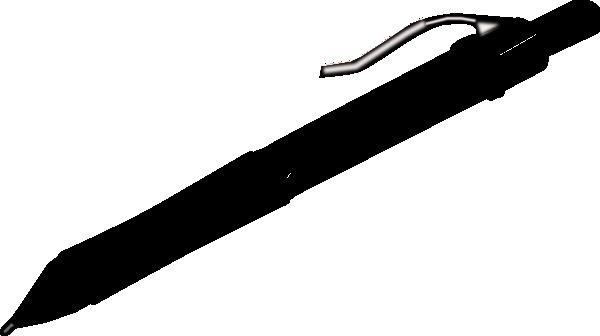}      
\end{subfigure}
\begin{subfigure}[b]{.13\columnwidth}
\includegraphics[width=1.2cm,height=1.2cm,keepaspectratio]{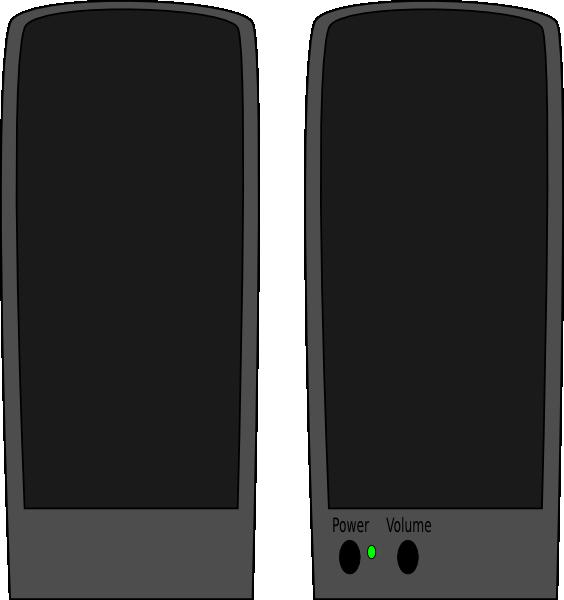}      
\end{subfigure}
\\
\begin{subfigure}[b]{.04\columnwidth}
\rotatebox{90}{\enskip \small{Product}}
\end{subfigure}
\begin{subfigure}[b]{.13\columnwidth}
\includegraphics[width=1.2cm,height=1.2cm,keepaspectratio]{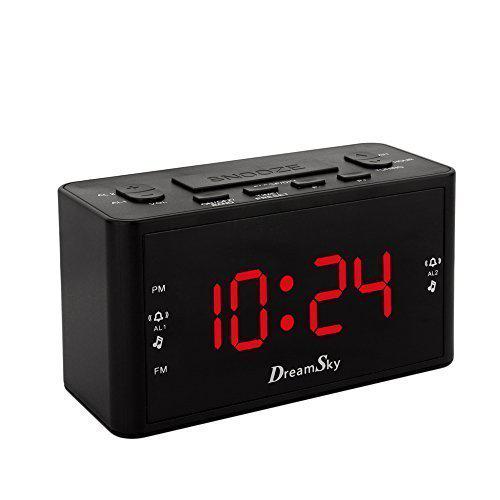}   
\end{subfigure}
\begin{subfigure}[b]{.13\columnwidth}
\includegraphics[width=1.2cm,height=1.2cm,keepaspectratio]{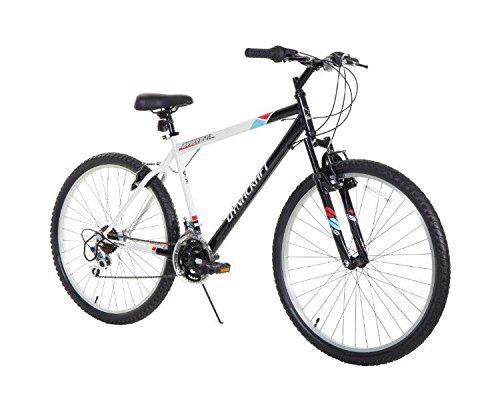}      
\end{subfigure}
\begin{subfigure}[b]{.13\columnwidth}
\includegraphics[width=1.2cm,height=1.2cm,keepaspectratio]{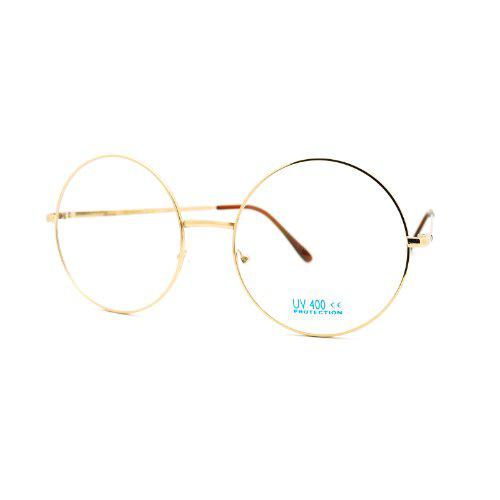}    
\end{subfigure}
\begin{subfigure}[b]{.13\columnwidth}
\includegraphics[width=1.2cm,height=1.2cm,keepaspectratio]{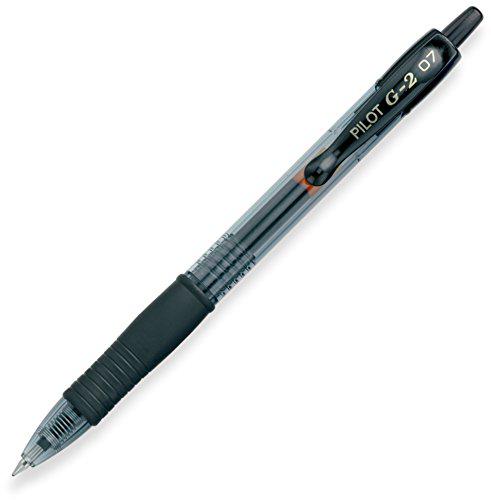}      
\end{subfigure}
\begin{subfigure}[b]{.13\columnwidth}
\includegraphics[width=1.2cm,height=1.2cm,keepaspectratio]{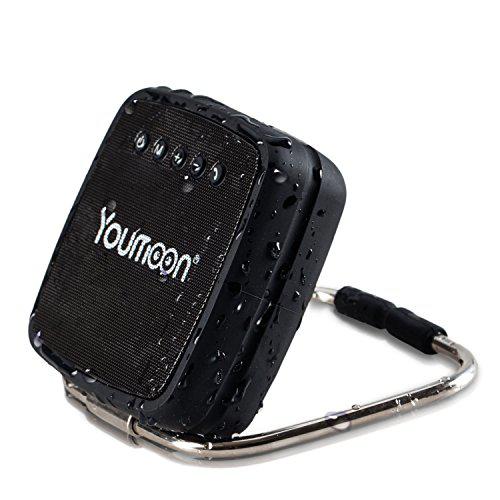}      
\end{subfigure}
\\
\begin{subfigure}[b]{.04\columnwidth}
\rotatebox{90}{\quad\quad \small{World}}
\end{subfigure}
\begin{subfigure}[b]{.13\columnwidth}
\includegraphics[width=1.2cm,height=1.2cm,keepaspectratio]{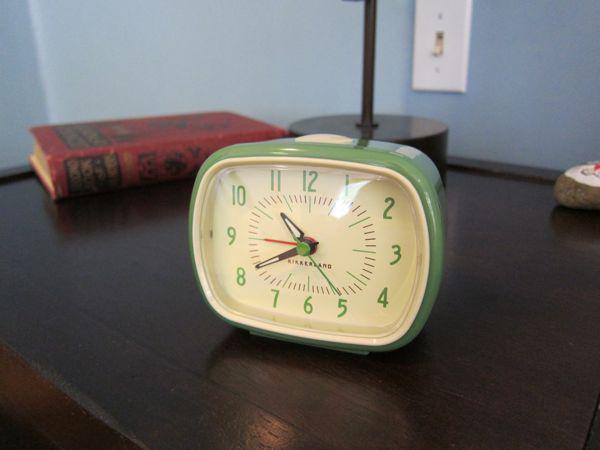}   
\caption{Alarm}
\end{subfigure}
\begin{subfigure}[b]{.13\columnwidth}
\includegraphics[width=1.2cm,height=1.2cm,keepaspectratio]{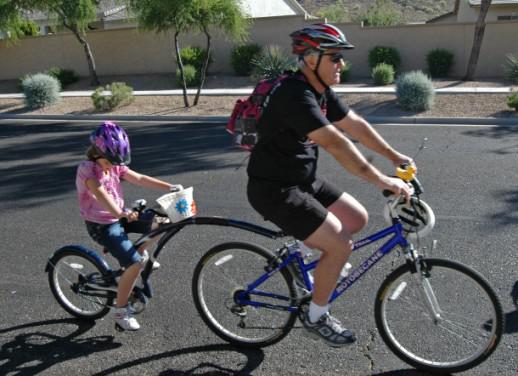}      
\caption{Bike}
\end{subfigure}
\begin{subfigure}[b]{.13\columnwidth}
\includegraphics[width=1.2cm,height=1.2cm,keepaspectratio]{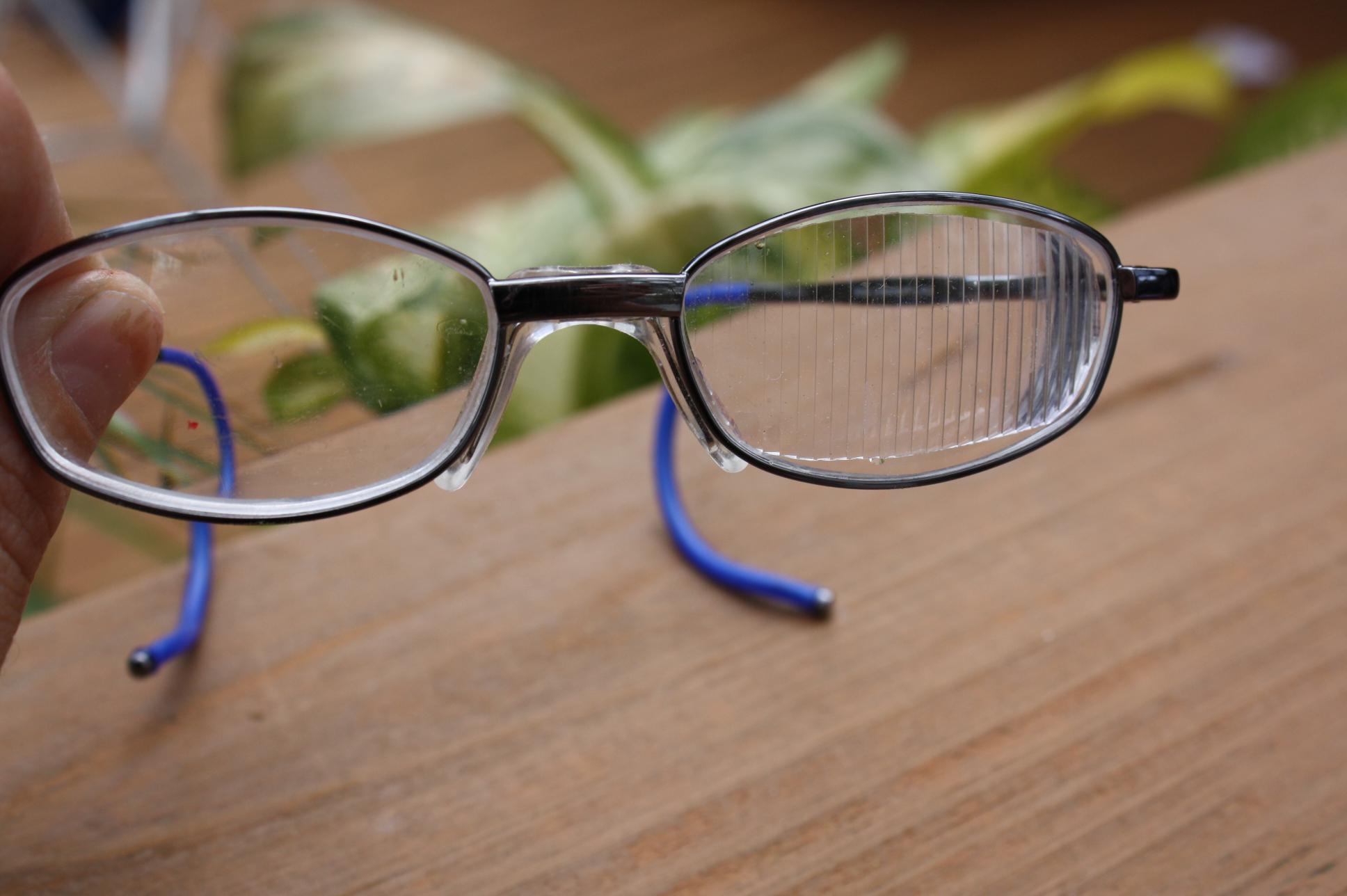}    
\caption{Glasses}
\end{subfigure}
\begin{subfigure}[b]{.13\columnwidth}
\includegraphics[width=1.2cm,height=1.2cm,keepaspectratio]{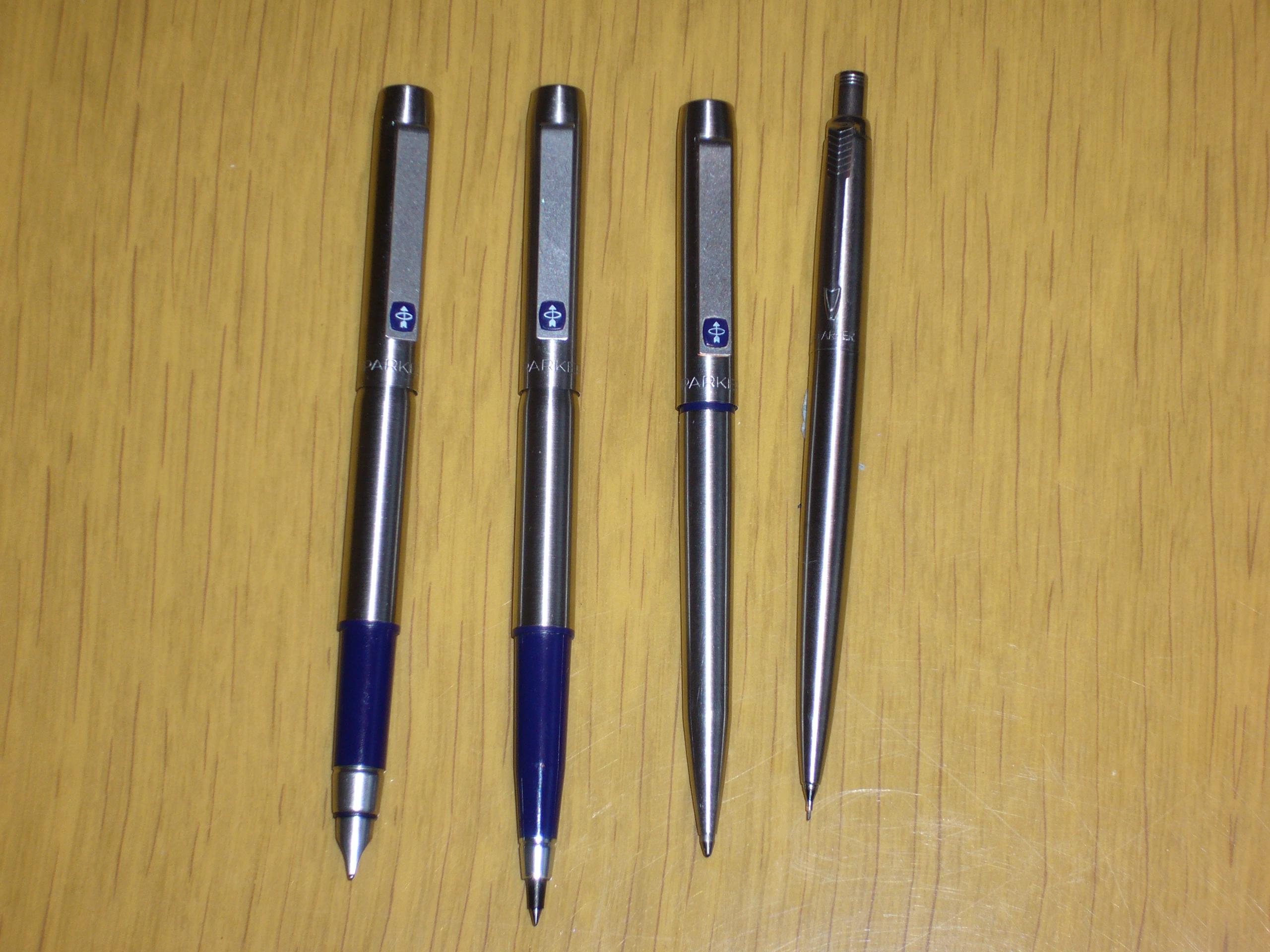}      
\caption{Pen}
\end{subfigure}
\begin{subfigure}[b]{.13\columnwidth}
\includegraphics[width=1.2cm,height=1.2cm,keepaspectratio]{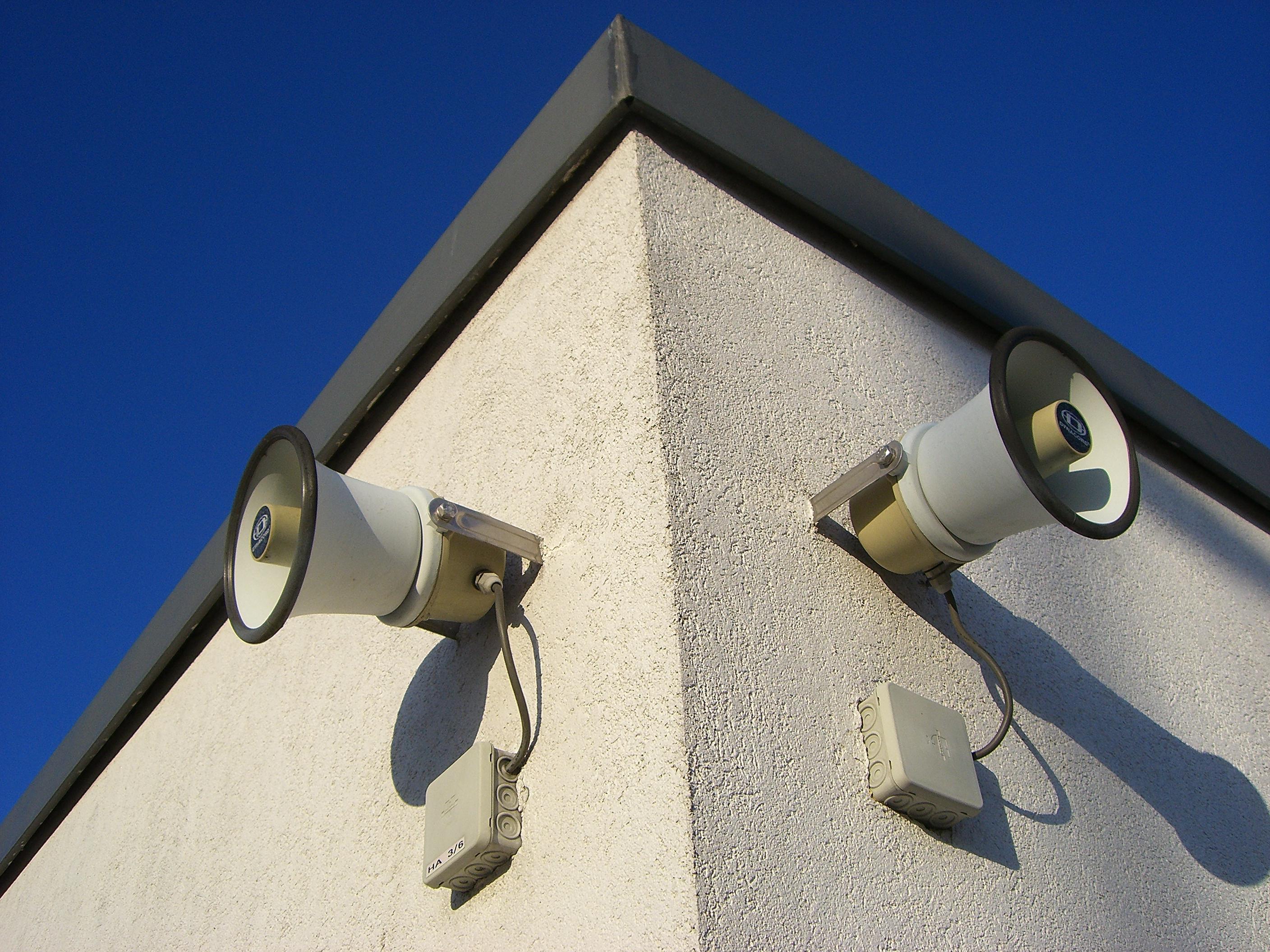}
\caption{Speaker}
\end{subfigure}
\caption{Example images from $5$ of the $65$ categories from the $4$ domains composing the Office-Home dataset \cite{venkateswara2017deep}. The $4$ domains are Art, Clipart, Product and Real World. In total, the dataset contains $15{,}500$ images of different sizes.}
\label{fig:Dataset}
\end{figure}

We show our results in Figure \ref{fig:accuracy}, where we use the classification accuracy of the trained networks on the test set as figure of merit.  
To begin with, we empirically corroborate that the dataset  share mutual information across domains, since the consensus  classifier achieves an accuracy of $35.59\%$,  outperforming the separate training whose accuracy is $
31.94\%$. This means that utilizing samples from other domains improves the performance of the learned classifier. This is the setting in which cross-learning is most helpful as it will help reduce bias. 
As it can be seen in Figure \ref{fig:accuracy}, a salient fact about the experiments is that the cross-learning estimator in both the parameter, as well as the output constraint, consistently outperforms both consensus and separate estimators, reaching a maximum accuracy of $44.5\%$  at $\epsilon=0.2$ when coupling the outputs. Indeed, in this scenario, the case with output constraints shows an improvement over parameter constraints.
However, in both cases the estimator has a positive slope coming from $\epsilon=0$, and a negative one when $\epsilon$ is big enough, attaining a maximum inbetween, which reproduces the theoretical findings of Section \ref{sec:perf_analysis}.

\begin{figure}[t]
\centering
\includegraphics[scale=1]{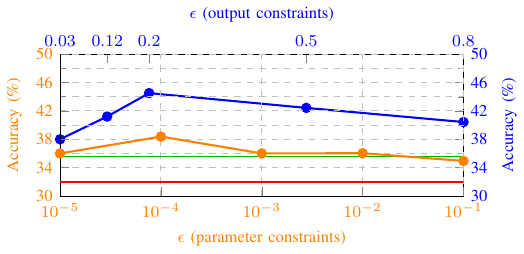}                
\vspace{-3ex}
\caption{Overall accuracy in percentage of correctly classified images measured on unseen data (test set). The consensus estimator (green) coincides with the parametric one (orange) when $\epsilon=0$ and the tends to the separate one (red) as $\epsilon\to\infty$. The output-constrained cross-learning estimator (blue) obtains the highest accuracy of 44.5\%. at $\epsilon=0.2$.
}
\label{fig:accuracy}
\end{figure}

As a byproduct, Algorithm \ref{alg:Algorithm} generates dual variables $\lambda_t$ that contain valuable information of the problem at hand.  Figure \ref{fig:lambdas} shows the  dual variables $\lambda$ as a function of the epoch. A larger dual variable indicates that the constraint is harder to satisfy, and a dual-variable equal to $0$ indicates that the constraint is inactive. As seen in \ref{fig:lambdas}, all $\lambda_t$ are non-negative, which means that we are effectively in a regime where there the classifiers are utilizing data from other tasks. 
If we look at the relative values of the dual variables, we see that the domain Art has the smallest value, whereas the domain Real World has the largest one. Given that Art has the least amount of samples, it is the domain that mostly benefits from images coming from other tasks. On the other hand, the largest dual variable is associated with the real-world dataset, have images with more details, textures, and shapes, and are therefore more difficult to classify (cf. Figure \ref{fig:Dataset}). 

\begin{figure}[t]
\centering
\includegraphics[scale=1]{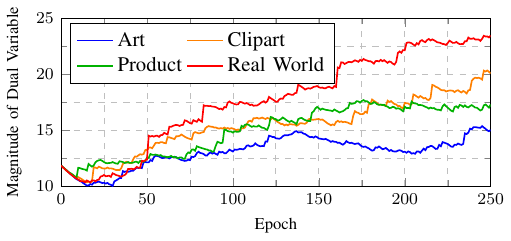}                
\caption{Dual variable associated with each constraint for the case of $\epsilon=0.12$. Intuitively, a larger dual variable indicates that the constraint is harder to satisfy.}
\label{fig:lambdas}
\end{figure}

\section{Conclusions}
\label{sec:Conclusions}
We proposed a constrained optimization method to achieve multi-task learning under a deterministic paradigm for the model parameters. Leveraging a controlled Gaussian model, we were able to prove that by forcing the parameter estimators to be close to one another, we obtain a strictly lower mean-squared error than if we use separate or consensus counterparts. Then we tested our cross-learning method on a practical problem with real data, where we fit parameters modeling the propagation of an infectious disease across different populations using a tailored algorithm based on proximal operators. The results obtained in this scenario replicated the theoretical ones, with cross-learning predicting the peak of infections with an error of $0.07\%$, compared to errors of $76.38\%$ when assuming that all populations follow the same propagation model, and $2474\%$ when estimating separate parameters. In the context of classification with neural networks, we proposed a variant of our cross-learning method that constrains the model outputs, rather than their parameters, to be at a close distance. Such a variant was also tested on real data, comprising images from different categories and domains. In this new test, the cross-learning estimator achieved a higher accuracy of $44.5\%$ when compared to $33.97\%$ and $24.44\%$ for the consensus and separate counterparts, respectively.

\appendices
\appendix

\subsection{Bias-Variance}\label{app:bias-variance}
Let us compute the bias and variance of the consensus estimator $\thetac=(1/T)\sum_{t=1}^T \thetahat$ as an estimator for the ground truth parameter $\thetastar$, under the model $\thetahat=\thetastar+\eta_t$  with zero-mean noise $\eta_t$ of variance $\sigma^2$.
The bias depends on the pairwise distances between the ground truth estimators, i.e.,  
\begin{align}
\mathbb{E}\left[\thetac-\thetastar\right]&=\mathbb{E}\left[\frac{1}{T}\sum_{\tau=1}^\top \left(\hat\theta_\tau-  \thetastar\right)\right]=\frac{1}{T}\sum_{\tau=1}^\top \left(\hat\theta_\tau-  \thetastar\right)\nonumber
\end{align}
so it introduces bias unless all $\thetastar$ coincide. On the other hand
\begin{align*}
\text{var}(\thetac)&=\mathbb{E}\left[\left\|\thetac-\mathbb{E}\left[\thetac\right]\right\|^2\right] =\mathbb{E}\left[\left\|\frac{1}{T}\sum_{t=1}^T\left( \thetahat-  \thetastar\right)\right\|^2\right]\\
&=\frac{1}{T}\sum_{t=1}^T \mathbb{E}\left[\left\|\left( \hat \theta_t-\theta^\star_t\right)\right\|^2\right]=
\frac{1}{T^2}\sum_{t=1}^{T}\frac{d}{N_t} \sigma^2.
\end{align*}
which is smaller, the larger the number of tasks $T$ is. 

\subsection{Auxiliary Proofs of Proposition \ref{prop:centralized_vs_cl}}\label{app:proofs}
We first establish that the cross-learning centroid is a convex combination of the separate estimates.
\begin{lemma}\label{lemma:thetaG}
Consider the centroid $\thetage$ solution to \eqref{eqn_perf_ana:cross_learning}, and 
the separate estimates $\thetahat=(1/N_t)\sum_{n=1}^{N_t}y_{nt}$. There exist a set of coefficients $\gamma_t\in[0,1]$ satisfying $\sum_{t=1}^T\gamma_t=1,$ such that
\begin{align}\thetage&=\sum_{t=1}^T \gamma_t \thetahat.\label{eq:convex_comb_lemma1}
\end{align}
\end{lemma}
\begin{proof}
The Lagrangian of  problem \eqref{eqn_perf_ana:cross_learning} takes the form
 \begin{align*}
L\left(\theta_t, \theta_g, \mu_t\right)&=\sum_{t=1}^T\left\|\thetahat-\theta_t\right\|^2+\sum_{t=1}^T\mu_t\left(\| \theta_t-\theta_g\|^2-\epsilon^2\right)
\end{align*}
Setting its gradient  with respect to $\theta_t$ to zero, we obtain
\begin{align}\label{eq:thetacl_lambda}
\thetacl&=\frac{\thetahat}{1+\mu_t}+\frac{\thetage \mu_t}{1+\mu_t} =\epsilon_t \thetahat+(1-\epsilon_t) \thetage
\end{align}
after defining $\epsilon_t=1/(1+\mu_t)$ so that $(1-\epsilon_t)=\mu_t/(1+\mu_t)$. On the other hand, setting the gradient of the Lagrangian with respect to $\theta_g$ to zero results in
\begin{align}\label{eq:derivative_Lagrangian_thetag}
    \sum_{t=1}^T\left(\thetage-\thetacl\right) \mu_t&=0.
\end{align}
Hence, we can substitute \eqref{eq:thetacl_lambda} in \eqref{eq:derivative_Lagrangian_thetag} to obtain
\begin{align}
\sum_{t=1}^T \mu_t \thetage &= \sum_{t=1}^T \mu_t \thetacl=\sum_{t=1}^T \mu_t\left(\left(\epsilon_t\right) \thetahat+(1-\epsilon_t) \thetage\right)\\
&=\sum \mu_t\epsilon_t \thetahat+\sum \mu_t (1-\epsilon_t) \thetage,
\end{align}
which we can  solve for $\thetage$ as in 
\begin{align}\label{eq:mulambdathetag}
 \thetage &= \sum_{t=1}^T \frac{\mu_t\epsilon_t \hat{\theta}_t}{\sum_{i=1}^T \mu_i\epsilon_t}= \sum_{t=1}^T \frac{(1-\epsilon_t) \hat{\theta}_t}{\sum_{i=1}^T (1-\epsilon_t)} . \end{align}
where we used $\mu_t\epsilon_t=\mu_t/(1+\mu_t)=(1-\epsilon_t)$, so that \eqref{eq:mulambdathetag} takes the  form  \eqref{eq:convex_comb_lemma1} 
if we identify  $\gamma_t\triangleq (1-\epsilon_t)/\sum_{i=1}^T (1-\epsilon_i).$
\end{proof}

\begin{lemma}\label{lemma:projection}
Consider the centroid $\thetage$ solution to \eqref{eqn_perf_ana:cross_learning}, and 
the separate estimates $\thetahat=(1/N_t)\sum_{n=1}^{N_t}y_{tn}$.
The cross-learning estimate $\thetacl$ in \eqref{eqn_perf_ana:cross_learning} is the projection of $\thetahat$ to the ball $\mathcal B(\thetage,\epsilon)$ of center $ \thetage$ and radius $\epsilon.$
\end{lemma}
\begin{proof}
By construction $\thetacl\in  \mathcal B(\thetage,\epsilon)$. 
 Furthermore,   according to \eqref{eq:thetacl_lambda} in the proof of Lemma \ref{lemma:thetaG},  the cross-learning estimate $\thetacl$ lies in a convex combination of $\thetage$ and  $\thetahat$.
For those constraints that are inactive, complementary slackness requires $\mu_t=0$. Substituting $\mu_t=0$ in \eqref{eq:thetacl_lambda} we obtain $\thetacl=\thetahat$, which means that $\thetahat$ also lies inside the ball, and hence they are projections of each other. If the constraint is active, then $\|\thetacl-\thetage\|=\epsilon$, which means that $\thetacl$ lies on the intersection of the  border of $\mathcal B(\thetage,\epsilon)$ with the segment between $\thetahat$ and $\thetage$, and that is the projection of $\thetahat$ to the ball again.  
\end{proof}

In the next Lemma we show that the consensus estimator $\thetac$ is close to the  centroid $\thetage$, which will be needed down the road to assess the performance of cross-learning.

\begin{lemma}\label{lemma:thetacthetag}
The error between the consensus estimator $\thetac$ and the cross-learning centroid $\thetage$ is given by
\begin{align}\thetage-\thetac=\frac{1}{T}\sum_{t=1}^T \epsilon_t \left(\thetac  -   \thetahat\right) +\mathcal O( \epsilon^2)\label{eq:lemma3tgetcerror}
\end{align}
with   $\epsilon_t = \epsilon/\|\thetahat-\thetac\|+\mathcal O(\epsilon^2)  =\mathcal O(\epsilon)$ in the limit as 
$\epsilon \to 0$.
\end{lemma}
\begin{proof}
     We know from Lemma \ref{lemma:thetaG} that
$\thetage=\sum_{t=1}^T \gamma_t \thetahat$
with
\begin{align}\gamma_t&=  \frac{(1-\epsilon_t)}{\sum_{k=1}^T (1-\epsilon_k)}= \frac{1}{T}\frac{(1-\epsilon_t)}{ (1-\varepsilon)}\label{eq:gamma_epsilonbar}
\end{align}
where $\varepsilon:=(1/T)\sum_{k=1}^T \epsilon_k$. 
A closer look at \eqref{eq:gamma_epsilonbar} reveals that
\begin{align} (1-\varepsilon)(1+\varepsilon)&= 1-\varepsilon^2 \Rightarrow \frac{1}{ (1-\varepsilon)}=1+\varepsilon +\mathcal O(\varepsilon^2).
\label{eq:order_quotientepsilonbar}
\end{align}
 Combining \eqref{eq:order_quotientepsilonbar} with \eqref{eq:gamma_epsilonbar} we obtain
\begin{align}\gamma_t&= \frac{1}{T}(1-\epsilon_t)(1+\varepsilon+\mathcal O( \epsilon^2))=\frac{1}{T}\left(1+\varepsilon- \epsilon_t+\mathcal O( \epsilon^2)\right),\nonumber
\end{align}
which can be substituted in $\thetage=\sum_{t=1}^T \gamma_t \thetahat$, to obtain
\begin{align}&\thetage-\thetac=\sum_{t=1}^T \left(\gamma_t - \frac{1}{T}\right) \thetahat=\frac{1}{T}\sum_{t=1}^T \left(\varepsilon -\epsilon_t+\mathcal O( \epsilon^2)\right) \thetahat\nonumber\\
&=\varepsilon\thetac\hspace{-2pt}-\hspace{-1pt}\frac{1}{T}\hspace{-2pt}\sum_{t=1}^T   \hspace{-1pt}\epsilon_t\thetahat\hspace{-2pt}+\hspace{-2pt}\mathcal O( \epsilon^2)\hspace{-1pt}
=\hspace{-1pt}\frac{1}{T}\hspace{-1pt}\sum_{t=1}^T\hspace{-1pt}\epsilon_t (\thetac\hspace{-2pt}-\hspace{-2pt}\thetahat)\hspace{-2pt}+\hspace{-2pt}\mathcal O( \epsilon^2).\nonumber
\end{align}
To finish the proof, we need to show that $\epsilon_t$ is of order $\epsilon$ as $\epsilon\to 0$. In this regime, we can assume without loss of generality that the constraints are active, and write 
\begin{align}
\epsilon&\hspace{-2pt}=\hspace{-2pt}\|\thetacl-\thetage\|\hspace{-2pt}=\hspace{-2pt}\|\epsilon_t \thetahat+(1-\epsilon_t) \thetage-\thetage\|\hspace{-2pt}=\hspace{-2pt}\epsilon_t\| \thetahat-\thetage\|\label{eq:epsilontorder}.
\end{align}
which implies that $\epsilon_t=\mathcal O(\epsilon)$. Still,  we need one more step to prove $\epsilon_t=\epsilon/\|\thetahat-\thetac\|+\mathcal O(\epsilon^2)$. So we rewrite \eqref{eq:epsilontorder} as
\begin{align}
\epsilon_t &= \epsilon \|\thetahat-\thetage\|^{-1}=\epsilon\left(\|\thetahat-\thetac+\thetac-\thetage\|^2\right)^{-1/2}\nonumber\\
&=\epsilon \left(\|\thetahat-\thetac\|^2+\|\thetac-\thetage\|^2+2(\thetahat-\thetac)^\top(\thetac-\thetage) \right)^{-1/2}\nonumber\\
&=\epsilon\|\thetahat-\thetac\|^{-1}\left(1+2 R_t\right)^{-1/2}
\end{align}
with $R_t= \left(\|\thetac-\thetage\|^2+(\thetac-\thetage)^\top(\thetahat-\thetac) \right)\|\thetahat-\thetac\|^{-2}.$
Since we already proved that $\thetac-\thetage=\mathcal O(\epsilon)$, we can write
\begin{align}\epsilon_t=\frac{\epsilon(1+2\mathcal O(\epsilon))^{-1/2}}{\|\thetahat-\thetac\|}=\frac{\epsilon(1-\mathcal O(\epsilon))}{\|\thetahat-\thetac\|}=\frac{\epsilon}{\|\thetahat-\thetac\|} +\mathcal O(\epsilon^2),\nonumber
\end{align}
which is the result that we wanted to prove.
\end{proof}
The next lemma provides an expression for the difference between the cross-learning sample error $\mathcal E_{CL}$  and its consensus counterpart $\mathcal E_{C}$ in the regime  $\epsilon\to 0$. Remarkably, this expression does not depends on the solution $\thetacl$ but on simpler variables that  we can model in closed form.
\begin{lemma}\label{lemma:gradientMSE}
Consider the sample errors  $\mathcal{E}_{C}$ and $\mathcal{E}_{CL}(\epsilon)$, resulting from the consensus and cross-learning estimators,  as defined in  \eqref{eqn_mseC_def} and \eqref{eqn_mseCL_def}, respectively. As $\epsilon\to 0$, $\mathcal{E}_{CL}(\epsilon)$ converges to $\mathcal{E}_C$
according to 
\begin{align}
    \mathcal E_{CL}(\epsilon)\hspace{-2pt}-\hspace{-2pt}\mathcal E_{C}\hspace{-2pt}=\hspace{-2pt}-2\epsilon \avt (\thetastar-\theta_c^\star)^\top \frac{\thetahat-\thetac}{\|\thetahat-\thetac\|}\hspace{-2pt}+\hspace{-2pt}\mathcal O(\epsilon^2),\label{eq:avg_a_x_avg_b}
\end{align}
where $\theta_c^\star=\frac{1}{T}\sum_{t=1}^T \thetastar$ averages  the ground truth parameters.  
\end{lemma}
\begin{proof}
Let start by comparing the square errors of the consensus and the cross-learning estimators 
    \begin{align}
        &\|\thetacl-\thetastar\|^2=\|\thetacl-\thetac+\thetac-\thetastar\|^2\\
        &=\|\thetac-\thetastar\|^2+\|\thetacl-\thetac\|^2- 2(\thetastar-\thetac)^\top(\thetacl-\thetac)\nonumber\\
        &=  \|\thetac-\thetastar\|^2- 2I_t+ \mathcal O(\epsilon^2) 
    \end{align}
    where we defined $I_t=(\thetacl-\thetac)^\top(\thetastar-\thetac)$, 
    and substituted $\|\thetacl-\thetac\|^2=\mathcal O(\epsilon^2)$ since  the cross-learning constraint imposes  $\|\thetacl-\thetage \|\leq \epsilon$ and  Lemma \ref{lemma:thetacthetag} implies $\|\thetage-\thetac \|=\mathcal O(\epsilon)$.  
    
To express $I_t$ in terms of $\thetahat$, we add and subtract $\thetage$ 
    \begin{align}
        I_t&=(\thetacl-\thetage+\thetage-\thetac)^\top(\thetastar-\thetac)\\
        &=(\epsilon_t(\thetahat-\thetage)+\thetage-\thetac)^\top(\thetastar-\thetac)\\
        &=\epsilon_t(\thetahat-\thetac)^\top(\thetastar-\thetac)+J_t+\mathcal O(\epsilon^2)
            \end{align}
where we defined $J_t=(\thetage-\thetac)^\top(\thetastar-\thetac)$ and substituted $\epsilon_t(\thetahat-\thetage)=\epsilon_t(\thetahat-\thetac)+\mathcal O(\epsilon^2)$ using Lemma \ref{lemma:thetacthetag} again.

Next, we average $J_t$  across tasks to obtain 

       \begin{align}
       \avt J_t&\hspace{-2pt}=\avt(\thetage-\thetac)^\top(\thetastar-\thetac) \hspace{-2pt}
    =\hspace{-2pt}(\thetage-\thetac)^\top( \theta_c^\star-\thetac)\nonumber\\
    &=\avt\epsilon_t(\thetac-\thetahat)^\top( \theta_c^\star-\thetac)+\mathcal O(\epsilon^2)\label{eq:averageJt}
    \end{align}
after substituting \eqref{eq:lemma3tgetcerror} for $(\thetage-\thetac)$.
Then we conclude the proof by averaging $I_t$, which results in
\begin{align}
       \avt I_t  &=\avt \epsilon_t(\thetahat-\thetac)^\top(\thetastar-\thetac)\\
       &+\avt\epsilon_t(\thetac-\thetahat)^\top( \theta_c^\star-\thetac)+\mathcal O(\epsilon^2)\\
       &=\avt (\thetastar-\theta_c^\star)^\top\frac{(\thetahat-\thetac)}{\|\thetahat-\thetac\|}+\mathcal O(\epsilon^2)
       \end{align}
       with $\epsilon_t=\epsilon/\|\thetahat-\thetac\| +\mathcal O(\epsilon^2)$  as in Lemma \ref{lemma:thetacthetag}.
\end{proof}

\bibliographystyle{IEEEtran}
\bibliography{bib}

\end{document}